\documentclass{amsart}

\usepackage[utf8]{inputenc} % allow utf-8 input
\usepackage[T1]{fontenc}    % use 8-bit T1 fonts
\usepackage[colorlinks=true, citecolor=myblue, linkcolor=black]{hyperref}
\usepackage{url}            % simple URL typesetting
\usepackage{booktabs}       % professional-quality tables
\usepackage{amsfonts}       % blackboard math symbols
\usepackage{nicefrac}       % compact symbols for 1/2, etc.
\usepackage{microtype}      % microtypography
\usepackage{xcolor}         % colors
\usepackage{mathtools}
\usepackage{mathabx}
\usepackage{booktabs}
\usepackage[T1]{fontenc}
\usepackage{amsmath,amsfonts,mathrsfs,amssymb,amsthm}
\usepackage{enumerate}
\usepackage{bm}
\usepackage{dsfont}
\usepackage{stmaryrd}
\usepackage{wrapfig,xfrac}
\usepackage{tikz}
\usepackage{tikz-3dplot}
\usetikzlibrary{arrows.meta,decorations.pathreplacing,positioning,calc}
\usepackage{pgfplots}
\pgfplotsset{compat=1.18}
\definecolor{myblue}{rgb}{0.21, 0.34, 0.74}
\definecolor{mygrey}{rgb}{0.55, 0.57, 0.67}
\definecolor{myred}{rgb}{0.79, 0.0, 0.09}
\definecolor{mygreen}{rgb}{0.05, 0.5, 0.06}

\DeclareMathAlphabet{\mathscrbf}{OMS}{mdugm}{b}{n}
\usepackage[scr=boondoxo]{mathalpha}
\usepackage{cleveref}
\usepackage{subcaption}
\DeclareCaptionLabelFormat{empty}{}
\usepackage{comment}

\usepackage{wrapfig}
\usepackage{upgreek}

\usepackage{units}

\renewcommand{\S}{\textsection}

\newcommand{\de}{\mathrm{d}}

\newcommand{\Id}{\mathrm{Id}}

\newcommand{\proj}{\mathbf{P}}
\newcommand{\dist}{\mathrm{dist}}

\renewcommand{\S}{\mathbb{S}^{d-1}}

\newcommand{\bone}{\mathbf{1}}

\newcommand{\Xp}{\mathcal{X}}

\newcommand*\diff{\mathop{}\!\mathrm{d}}
\newcommand{\R}{{\rm I}\kern-0.18em{\rm R}}
\newcommand{\h}{{\rm I}\kern-0.18em{\rm H}}
\newcommand{\K}{{\rm I}\kern-0.18em{\rm K}}
\newcommand{\p}{{\rm I}\kern-0.18em{\rm P}}
\newcommand{\E}{{\rm I}\kern-0.18em{\rm E}}

\newcommand{\1}{{\rm 1}\kern-0.24em{\rm I}}
\newcommand{\N}{{\rm I}\kern-0.18em{\rm N}}

\newcommand{\x}{\mathcal{X}}

\newcommand{\hh}{{\sf h}}

\newcommand{\att}{\mathsf{\mathop{ATT}}}

\newcommand{\Np}{\mathcal{N}}

\numberwithin{equation}{section}

\renewcommand{\paragraph}[1]{%
  \par\smallskip
  \noindent\textbf{#1}\hspace{1em}\ignorespaces
}

\usepackage[font=small,labelfont=bf]{caption}

\theoremstyle{plain}
\newtheorem{theorem}{Theorem}[section]
\newtheorem{assumption}{Assumption}

\newtheorem{proposition}[theorem]{Proposition}
\newtheorem{lemma}[theorem]{Lemma}
\newtheorem{corollary}[theorem]{Corollary}

\newtheorem{remark}[theorem]{Remark}

\begin{document}

\title{Scaling Limits of Long-Context Transformers}

\author[G. Bruno]{Giuseppe Bruno}
%    Address of record for the research reported here
\address{(GB) Departement Mathematik und Statistik, University of Bern, Sidlerstrasse 5, 3012 Bern, Switzerland} %    Current address
\email{giuseppe.bruno@unibe.ch}

\author[S. Chen]{Shi Chen}
%    Address of record for the research reported here
\address{(SC) Department of Mathematics, Massachusetts Institute of Technology, 77 Massachusetts Ave, 02139 Cambridge MA, USA} %    Current address
\email{schen636@mit.edu}

\author[Z. Lin]{Zhengjiang Lin}
%    Address of record for the research reported here
\address{(ZL) Department of Mathematics, Massachusetts Institute of Technology, 77 Massachusetts Ave, 02139 Cambridge MA, USA} %    Current address
\email{linzj@mit.edu}

\author[Y. Polyanskiy]{Yury Polyanskiy}
\address{(YP) Department of Electrical Engineering and Computer Science, Massachusetts Institute of Technology, 77 Massachusetts Ave, 02139 Cambridge MA, USA}
\email{yp@mit.edu}

\author[P. Rigollet]{Philippe Rigollet}
%    Address of record for the research reported here
\address{(PR) Department of Mathematics, Massachusetts Institute of Technology, 77 Massachusetts Ave, 02139 Cambridge MA, USA} %    Current address
\email{rigollet@math.mit.edu}

\date{}


\begin{abstract}
We study the long-context limit of softmax self-attention with a fixed query and a random context of $n$ i.i.d.\ keys on the sphere with the intrinsic dimension of the key distribution held fixed, viewing the inverse temperature $\beta_n$ as the scaling parameter that decides whether attention degenerates into uniform averaging or collapses onto the single closest key. We show that the critical scale at which selectivity emerges is determined by the local exponent of the distance-to-query distribution near zero rather than by global features of the context, and scales like $\beta_n^\ast \asymp n^{2/(d-1)}$ for uniform keys on $\S$. Furthermore, we characterize the limiting laws of the ordered attention weights and of the attention output across all regimes of $\beta_n$: a subcritical regime in which the output reduces to a local average around $q$ with explicit deterministic bias and Gaussian fluctuations; a critical regime in which a finite collection of nearest keys retains macroscopic mass without single-key collapse; and a supercritical regime in which all mass concentrates on the closest key. Of notable interest is the subcritical case with identity value matrix where the attention map approximately implements a backward heat equation.
\end{abstract}

\maketitle

\setcounter{tocdepth}{2}
\makeatletter
\def\l@subsection{\@tocline{2}{0pt}{2.8pc}{5pc}{}}

\tableofcontents

%-------------------------------------------------

%%%%%%%%%%%%%%%%%%%%%%%%%%%%%%%%%%%%%%%

\section{Introduction}\label{sec:introduction}

The push toward ever longer contexts has placed transformers on a trajectory in which the context length $n$ is no longer a fixed architectural detail but a scaling parameter in its own right~\cite{dai2019transformer,beltagy2020longformer,zaheer2020big,gemini}. From this vantage point, a finite-context attention head is naturally understood as a finite-sample approximation of an idealized \emph{infinite-context} limit. Context length then plays the role of a resource controlling how faithfully the finite system resolves this idealization, and its $n\to\infty$ scaling limit becomes the natural object to study.

Self-attention~\cite{vaswani2017attention} with Post-LayerNorm\footnote{We show how PostLN attention with general Key and Query matrices reduces to this model in Section~\ref{sec:main model and assumptions}.} routes a query $q\in\S$ against a context of keys $x_1,\dots,x_n\in\S$ through the softmax attention weights
\begin{equation}\label{eq:attention-weights-intro}
A_j \;=\; \frac{\exp\bigl(\beta_n \langle q,x_j\rangle\bigr)}{\sum_{l=1}^n \exp\bigl(\beta_n \langle q,x_l\rangle\bigr)},
\qquad j=1,\dots,n,
\end{equation}
which form a probability distribution over the context and are then used to aggregate the $x_j$. The inverse temperature $\beta_n$ is a classical device employed in long-context transformers~\cite{bai2023qwen,peng2023yarn, zhang2024selective, anson2025scale,bruno2025multiscale,nakanishi2025scalable,puvvada2025swan,chen2025critical,giorlandino2025failuremodesdeeptransformers,qu2026tabiclv2} that tunes how sharply $A=(A_1,\dots,A_n)$ separates keys with highest scores from the bulk.   In all settings,  $\beta_n$ is set to grow with the context length $n$. With no rescaling ($\beta_n=1$), even a key whose alignment $\langle q,x_j\rangle$ exceeds the bulk by a constant margin still produces a weight $A_j = \Theta(1/n)$: the exponential gap such a key introduces is dwarfed by the $n$ competing terms in the softmax denominator. The head loses the ability to distinguish relevant keys from the bulk---a phenomenon now commonly called \emph{attention dilation} or \emph{attention fading}~\cite{velivckovic2024softmax,nakanishi2025scalable}---and $\beta_n$ must grow with $n$ at the right rate to preserve selectivity. 

A first wave of theoretical work has examined the scaling limits as $n \to \infty, \beta_n\to \infty$ through simple, fully tractable models in which the fluctuations across keys are suppressed by design. These analyses connect attention sharpness to the contraction properties of self-attention layers~\cite{geshkovski2023mathematical,chen2025quantitative} and identify a logarithmic critical scale~\cite{nakanishi2025scalable,chen2025critical}: when $\beta_n \asymp \log n$, attention concentrates on $\Theta(1)$ tokens while remaining content-adaptive. Such results pin down the right scale, but the price of tractability is that the random objects which actually live at this scale---the size of the largest weight, the gap to the runner-up, the effective number of attended tokens, and the corresponding fluctuations of the attention map---are by construction invisible. A finer description requires that the keys themselves be treated as random and that the entire random profile of attention be analyzed at large $n$.

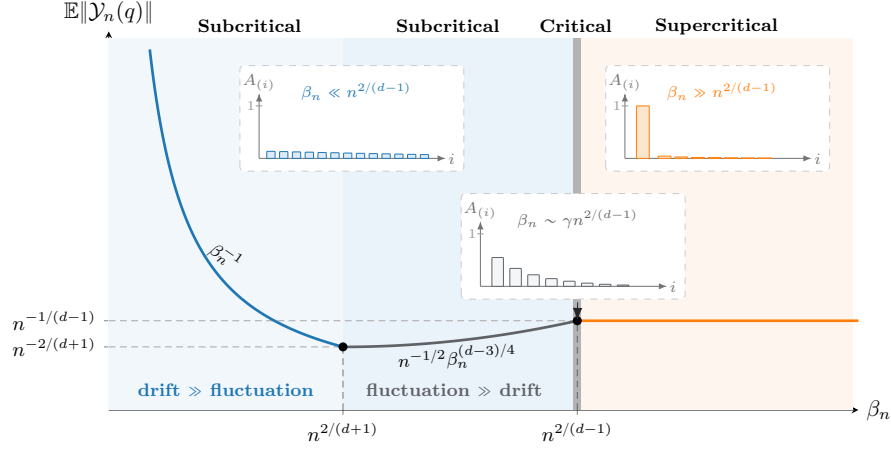
\begin{figure}[t]
\centering

% -----------------------------------------------------------------------------
% Colors adapted from the uploaded ordered-weight histogram figure.
% -----------------------------------------------------------------------------
\definecolor{CBBlueFill}{RGB}{222,235,247}
\definecolor{CBBlueFillAlt}{RGB}{206,226,242}
\definecolor{CBBlueDraw}{RGB}{31,119,180}
\definecolor{CBOrangeFill}{RGB}{252,230,206}
\definecolor{CBOrangeDraw}{RGB}{255,127,14}
\definecolor{CBNeutralFill}{RGB}{245,246,248}
\definecolor{CBNeutralDraw}{RGB}{95,99,104}

\colorlet{OWSubFill}{CBBlueFill}
\colorlet{OWSubDraw}{CBBlueDraw}
\colorlet{OWCritFill}{CBNeutralFill}
\colorlet{OWCritDraw}{CBNeutralDraw}
\colorlet{OWSupFill}{CBOrangeFill}
\colorlet{OWSupDraw}{CBOrangeDraw}

% Compact histogram panels so they fit cleanly above the main plot.
\newcommand{\OWPanelScale}{1.00}

\newcommand{\OWSubcriticalPanel}{%
\begin{tikzpicture}[
  scale=\OWPanelScale,
  transform shape,
  >=Latex,
  panel/.style={draw=black!20, fill=white, rounded corners=2pt, line width=0.45pt, dashed},
  axis/.style={-{Latex[length=1.3mm]}, black!55, line width=0.42pt},
  tick/.style={black!40, line width=0.32pt},
  note/.style={font=\scriptsize, align=center},
  title/.style={font=\small\bfseries, align=center},
  bar/.style={line width=0.2pt, rounded corners=0.3pt}
]
  \filldraw[panel] (-0.2,-0.2) rectangle (3.3,1.45);
  \node[note, text=OWSubDraw] at (1.62,0.9) {$\beta_n\ll n^{2/(d-1)}$};

  \draw[axis] (0.10,0) -- (3.02,0) node[right=1pt, font=\scriptsize] {$i$};
  \draw[axis] (0.10,0) -- (0.10,1.02) node[above=1pt, font=\scriptsize] {$A_{(i)}$};
  \draw[tick] (0.05,0.82) -- (0.15,0.82) node[left=3pt, font=\tiny] {$1$};

  \foreach \x/\h in {
    0.22/0.11,0.42/0.106,0.62/0.101,0.82/0.097,
    1.02/0.092,1.22/0.088,1.42/0.083,1.62/0.079,
    1.82/0.075,2.02/0.071,2.22/0.067,2.42/0.063,
    2.62/0.059
  } {
    \filldraw[bar, fill=OWSubFill, draw=OWSubDraw] (\x,0) rectangle ++(0.13,\h);
  }

\end{tikzpicture}%
}

\newcommand{\OWCriticalPanel}{%
\begin{tikzpicture}[
  scale=\OWPanelScale,
  transform shape,
  >=Latex,
  panel/.style={draw=black!20, fill=white, rounded corners=2pt, line width=0.45pt, dashed},
  axis/.style={-{Latex[length=1.3mm]}, black!55, line width=0.42pt},
  tick/.style={black!40, line width=0.32pt},
  note/.style={font=\scriptsize, align=center},
  title/.style={font=\small\bfseries, align=center},
  bar/.style={line width=0.2pt, rounded corners=0.3pt}
]
  \filldraw[panel] (-0.2,-0.2) rectangle (3.3,1.45);
  \node[note, text=OWCritDraw] at (1.62,0.9) {$\beta_n\sim\gamma n^{2/(d-1)}$};

  \draw[axis] (0.10,0) -- (3.02,0) node[right=1pt, font=\scriptsize] {$i$};
  \draw[axis] (0.10,0) -- (0.10,1.02) node[above=1pt, font=\scriptsize] {$A_{(i)}$};
  \draw[tick] (0.05,0.82) -- (0.15,0.82) node[left=3pt, font=\tiny] {$1$};

  \foreach \x/\h in {
    0.28/0.45,0.56/0.28,0.84/0.18,1.12/0.12,
    1.40/0.08,1.68/0.05,1.96/0.032,2.24/0.020
  } {
    \filldraw[bar, fill=OWCritFill, draw=OWCritDraw] (\x,0) rectangle ++(0.18,\h);
  }

\end{tikzpicture}%
}

\newcommand{\OWSupercriticalPanel}{%
\begin{tikzpicture}[
  scale=\OWPanelScale,
  transform shape,
  >=Latex,
  panel/.style={draw=black!20, fill=white, rounded corners=2pt, line width=0.45pt, dashed},
  axis/.style={-{Latex[length=1.3mm]}, black!55, line width=0.42pt},
  tick/.style={black!40, line width=0.32pt},
  note/.style={font=\scriptsize, align=center},
  title/.style={font=\small\bfseries, align=center},
  bar/.style={line width=0.2pt, rounded corners=0.3pt}
]
  \filldraw[panel] (-0.2,-0.2) rectangle (3.3,1.45);
  \node[note, text=OWSupDraw] at (1.62,0.9) {$\beta_n\gg n^{2/(d-1)}$};

  \draw[axis] (0.10,0) -- (3.02,0) node[right=1pt, font=\scriptsize] {$i$};
  \draw[axis] (0.10,0) -- (0.10,1.02) node[above=1pt, font=\scriptsize] {$A_{(i)}$};
  \draw[tick] (0.05,0.82) -- (0.15,0.82) node[left=3pt, font=\tiny] {$1$};

  \foreach \x/\h in {
    0.28/0.82,0.62/0.035,0.88/0.022,1.14/0.014,
    1.40/0.010,1.66/0.007,1.92/0.005,2.18/0.0035
  } {
    \filldraw[bar, fill=OWSupFill, draw=OWSupDraw] (\x,0) rectangle ++(0.20,\h);
  }

\end{tikzpicture}%
}

\resizebox{0.95\textwidth}{!}{
\begin{tikzpicture}

% -----------------------------------------------------------------------------
% Concrete parameters for the scaling curve.
% The y-axis is placed at x=0, while the 1/beta_n curve starts at a small
% positive beta value so it does not intersect the axis or blow up.
% -----------------------------------------------------------------------------
\pgfmathsetmacro{\nval}{64}
\pgfmathsetmacro{\alphaval}{0.5}
\pgfmathsetmacro{\betaA}{pow(\nval,\alphaval/(1+\alphaval))}
\pgfmathsetmacro{\betaB}{pow(\nval,\alphaval)}
\pgfmathsetmacro{\midexp}{1 - 1/\alphaval}
\pgfmathsetmacro{\yA}{1/\betaA}
\pgfmathsetmacro{\yB}{pow(\nval,-\alphaval/2)}
\pgfmathsetmacro{\axisxmin}{0}
\pgfmathsetmacro{\curvexmin}{0.7}
\pgfmathsetmacro{\xmax}{12.8}
\pgfmathsetmacro{\ymax}{1.50}
\pgfmathsetmacro{\critLeft}{\betaB-0.06}
\pgfmathsetmacro{\critRight}{\betaB+0.06}

\begin{axis}[
    name=phaseplot,
    width=13.3cm,
    height=7.5cm,
    xmin=\axisxmin, xmax=\xmax,
    ymin=0, ymax=\ymax,
    axis lines=left,
    xlabel={$\beta_n$},
    ylabel={$\mathbb{E}\|\mathcal Y_n(q)\|$},
    clip=false,
    xtick={\betaA,\betaB},
    xticklabels={
      $n^{2/(d+1)}$,
      $n^{2/(d-1)}$
    },
    ytick={\yA,\yB},
    yticklabels={
      $n^{-2/(d+1)}$,
      $n^{-1/(d-1)}$
    },
    tick label style={font=\small},
    label style={font=\small},
    every axis x label/.style={at={(ticklabel* cs:1)}, anchor=west},
    every axis y label/.style={at={(ticklabel* cs:1)}, anchor=south},
]

% Background regions.
\fill[CBBlueFill!40]    (axis cs:\axisxmin,0) rectangle (axis cs:\betaA,\ymax-0.03pt);
\fill[CBBlueFillAlt!40] (axis cs:\betaA,0) rectangle (axis cs:\betaB,\ymax-0.03pt);
\fill[CBOrangeFill!35]  (axis cs:\betaB,0) rectangle (axis cs:\xmax-0.1pt,\ymax-0.03pt);
\fill[black!30]    (axis cs:\critLeft,0) rectangle (axis cs:\critRight,\ymax-0.03pt);

% Crossover lines.
\draw[densely dashed, black!55]
  (axis cs:\betaA,0) -- (axis cs:\betaA,\yA);
\draw[densely dashed, CBNeutralDraw]
  (axis cs:\betaB,0) -- (axis cs:\betaB,\yB);

% Horizontal guides.
\draw[densely dashed, black!35]
  (axis cs:\axisxmin,\yA) -- (axis cs:\betaA,\yA);
\draw[densely dashed, black!35]
  (axis cs:\axisxmin,\yB) -- (axis cs:\xmax,\yB);

% Actual piecewise curve. The left curve begins at curvexmin > 0.
\addplot[
    very thick,
    CBBlueDraw,
    domain=\curvexmin:\betaA,
    samples=200
] {1/x};

\addplot[
    very thick,
    CBNeutralDraw,
    domain=\betaA:\betaB,
    samples=200
] {pow(x-\betaA,2)*(pow(\nval*pow(\betaB,\midexp),-0.5)-pow(\nval*pow(\betaA,\midexp),-0.5))/pow(\betaB-\betaA,2)+pow(\nval*pow(\betaA,\midexp),-0.5)};

\addplot[
    very thick,
    CBOrangeDraw,
    domain=\betaB:\xmax,
    samples=2
] {\yB};

% Crossover points.
\fill[black] (axis cs:\betaA,\yA) circle (2pt);
\fill[black] (axis cs:\betaB,\yB) circle (2pt);

% Labels for the three scaling pieces.
\node[font=\footnotesize, rotate=-35] at (axis cs:2,0.61)
    {$\beta_n^{-1}$};
\node[font=\footnotesize, rotate=2, align=center] at (axis cs:5.95,0.21)
    {$n^{-1/2}\beta_n^{(d-3)/4}$};

% Region labels.
\node[font=\footnotesize\bfseries, align=center, text=CBBlueDraw]
    at (axis cs:2.0,0.08) {drift $\gg$ fluctuation};
\node[font=\footnotesize\bfseries, align=center, text=CBNeutralDraw]
    at (axis cs:5.9,0.08) {fluctuation $\gg$ drift};
    
% Anchor points for the external histogram panels.
\coordinate (subpaneltarget)  at (axis cs:2.10,0.84);
\coordinate (critpaneltarget) at (axis cs:\betaB,\yB);
\coordinate (suppaneltarget)  at (axis cs:9.15,0.69);

\end{axis}

% External histogram panels.
\node[inner sep=0pt, anchor=south] (subpanel)
  at ($(phaseplot.north west)+(3.80cm,-2.2cm)$) {\OWSubcriticalPanel};

\node[inner sep=0pt, anchor=south] (critpanel)
  at ($(phaseplot.north)+(1.40cm,-4.2cm)$) {\OWCriticalPanel};

\node[inner sep=0pt, anchor=south] (suppanel)
  at ($(phaseplot.north east)+(-2.2cm,-2.2cm)$) {\OWSupercriticalPanel};

\node[font=\footnotesize\bfseries] at (2.2,6) {Subcritical};
\node[font=\footnotesize\bfseries] at (5.3,6) {Subcritical};
\node[font=\footnotesize\bfseries] at (7.3,6) {Critical};
\node[font=\footnotesize\bfseries] at (9.5,6) {Supercritical};

% Connection arrows from panels to the corresponding regions.
\draw[densely dashed, -{Latex[length=2mm]}, black!85]
  ([xshift=2pt,yshift=-1pt]critpanel.south) -- (critpaneltarget);

\end{tikzpicture}
}

\caption{Unified schematic of the ordered-weight and attention-output regimes as the inverse temperature $\beta_n$ varies, with $\mathcal Y_n(q)=\att (n)-Vq$}
\label{fig:scaling-regime-map}
\end{figure}

Probability theory has a long tradition of answering such questions: take a random system indexed by a size $n$, find the rescaling under which it admits a non-degenerate limit, and identify the law that emerges. Familiar instances include Gaussian fluctuations at scale $\sqrt{n}$ in the central limit theorem, the Fisher--Tippett--Gnedenko classification of extremal laws, and the $\log n$-scale free-energy transition of the Random Energy Model~\cite{Derrida1981}, recently connected to long-context attention by~\cite{giorlandino2025failuremodesdeeptransformers}.

For a fixed query vector $q\in\S$ and context vectors $x_1,\dots,x_n\in\S$, define the \emph{attention weights} as in~\eqref{eq:attention-weights-intro}
and the corresponding \emph{attention output}
\[
\att(n)=V\sum_{j=1}^n A_j x_j,
\]
where $V\in\R^{d\times d}$ is a fixed \emph{value matrix}. The natural scalar coordinate for these objects is the \emph{distance-to-query}, or D2Q,
\(
T_i\coloneqq 1-\langle q,x_i\rangle,
\)
which is half the squared Euclidean distance from key $x_i$ to the query. Two qualitatively distinct failure modes flank the regime of interest: if $\beta_n$ is too small, the smallest D2Qs are blurred into the bulk, the weights $A_j$ are essentially uniform, and the attention map fails to distinguish relevant tokens; if $\beta_n$ is too large, the entire output collapses onto the single closest key, and the map ceases to integrate the context. The question addressed in this paper is:
\begin{quote}
\emph{As $n\to\infty$, which scalings of $\beta_n$ produce non-degenerate limiting laws for the ordered attention weights $\{A_j\}_j$ and the attention output $\att(n)$, and what are those laws?}
\end{quote}

To make the limit precise, we model the keys $x_1,\dots,x_n$ as i.i.d.\ samples from a common distribution on the sphere, which induces a distribution of the D2Qs $T_1,\dots,T_n$ on $[0,2]$. The behavior at large $n$ is then governed by the local behavior of this distribution near $0$---the distribution of close neighbors---through the same mechanism that controls right-tail extremes in classical extreme-value theory. Within this model, our results below paint a complete picture of the scaling limits of attention across regimes of $\beta_n$, with explicit limit theorems at and around the critical scale.

\medskip
\noindent{\bf Our contributions.} 
We answer the question above through a statistical scaling theory for the attention distribution of a fixed query $q$ over a random context. The unifying principle is that long-context behavior is controlled not by the global shape of the context distribution but by the small-$T$ tail of the D2Q---i.e., the local geometry of the context around $q$---and by the point process of the closest keys that this tail generates. Our main results are the following.

\begin{itemize}

    \item[\textbf{(C1)}] \emph{The critical scale is set by the small-D2Q tail.}
    We identify the inverse temperature $\beta_n^*$ at which softmax first resolves the closest keys to $q$ from the bulk of the context, and show that this scale is determined by the local exponent of the D2Q distribution near zero, rather than by global features of the context. For keys uniform on $\S$, this principle yields the polynomial law $\beta_n^*\asymp n^{2/(d-1)}$, and for more structured key distributions it singles out local dimension and density as the relevant quantities.

    \item[\textbf{(C2)}] \emph{Limit laws for the ordered attention weights.}
    We characterize the asymptotic law of the ordered profile $A_{(1)}\ge A_{(2)}\ge\cdots$ across the full range of $\beta_n$. Three regimes emerge within a single framework: a subcritical regime in which the weights flatten and the head approaches uniform averaging; a critical regime around $\beta_n^*$ at which the weights converge to a non-degenerate random profile, with a finite collection of influential neighbors retaining macroscopic mass; and a supercritical regime in which the weight on the single closest key tends to one. The critical limit is the principal object of the paper: it is the unique scale at which selectivity emerges without single-key collapse.

    \item[\textbf{(C3)}] \emph{Limit laws for the attention output.}
    We extend the analysis to the attention output $\att(n)=V\sum_j A_j x_j$. The subcritical phase is not merely a regime of vanishing individual weights: as long as $\beta_n$ still diverges, $\att(n)$ behaves as a local average of the keys around $q$, and we obtain its limiting law as the sum of a deterministic bias, a Gaussian fluctuation term, and a density-sensitive tangential drift. At and above the critical scale, the corresponding limits of $\att(n)$ follow from those of the weights in (C2).

    \item[\textbf{(C4)}] \emph{A heat-type interpretation of subcritical residual updates.}
    The deterministic vector field that drives the subcritical output in (C3) admits a transparent interpretation: in the isotropic case, its tangential component is the gradient of the log token density. Attention layers operating below the critical scale therefore push representations from low-density toward high-density regions of the context, which corresponds approximately to a backward heat evolution.
\end{itemize}

\paragraph{Related work.} Empirical work on long-context attention provides much of the motivation for the present study.  A number of recent methods modify the attention logits in order to preserve useful sparsity as the context grows, including Qwen-style scaling, YaRN, SWAN-GPT, scalable-softmax, selective attention, and scale-invariant attention~\cite{bai2023qwen,peng2023yarn,puvvada2025swan,nakanishi2025scalable,zhang2024selective,anson2025scale}.  These works show that attention sharpness is an important practical degree of freedom in long contexts.  Our goal is complementary: we do not propose a new implementation, but instead ask what limiting law a rescaled attention row should obey once the context itself is modeled statistically.

Theoretical explanations of such scaling laws are still limited.  The deterministic critical-scaling theory of~\cite{chiang2022overcoming,nakanishi2025scalable,chen2025critical} identifies a logarithmic scale $\beta_n\asymp\log n$, while the signal-propagation theory of~\cite{giorlandino2025failuremodesdeeptransformers} and the scale-invariant-attention framework of~\cite{anson2025scale} point, in different Gaussian settings, to $\sqrt{\log n}$-type score scales.  Our polynomial scale arises from a different score geometry: the relevant competitors are rare nearest neighbors to a fixed query vector, and the scale is set by the local lower tail of the D2Q.  
Thus the contrast between logarithmic, square-root-logarithmic, and polynomial laws should be viewed as evidence that attention scaling depends on the geometry of the competing scores, rather than as a disagreement between models. 
Concurrent work~\cite{hayase2026unified} estimates temperature scaling
in trained models via score-gap counting, while we derive
limiting laws from a probabilistic model of the key distribution.
After submitting this manuscript, we became aware of an independent heuristic observation of the $n^{2/(d-1)}$ scaling for spherical keys in~\cite{nor2025longcontext}.

\paragraph{Notation.} We introduce the notation used throughout the paper.  
We denote the ordered attention scores as $A_{(1)} \ge A_{(2)} \ge \dots \ge A_{(n)}$.
Define the \emph{D2Q}
$T_i\coloneqq 1-\langle q,x_i\rangle\in[0,2]$, $i=1,\ldots,n$.
The attention weights can be written as
\begin{equation}
    \label{eq:weights_defect_form}
    A_i=\frac{e^{-\beta_nT_i}}{Z_n},
    \qquad
    Z_n\coloneqq\sum_{j=1}^n e^{-\beta_nT_j}.
\end{equation}
Let $T_{(1)}\le T_{(2)}\le\cdots\le T_{(n)}$ be the order statistics of the D2Qs.  
Finally, let the D2Q CDF be $F(t)\coloneqq \mathbb P(T_1\le t)$, $t\in[0,2]$. We set $\alpha\coloneqq \tfrac{2}{d-1}$ throughout the paper. 
We define the constant $C(q)$ to be the local intensity of near matches to the query $C(q) \coloneqq \lim_{t\to 0^+} F(t) t^{-1/\alpha}$.

\paragraph{Organization.} The remainder of the paper is organized as follows.
\Cref{sec:main model and assumptions} introduces the fixed-query attention model and the spherical i.i.d.\ baseline.
\Cref{sec:main results} establishes the subcritical, critical, and supercritical limits for the ordered attention weights.
\Cref{sec:attention_outputs} derives the corresponding attention-output limits, including the marked Poisson functional at criticality and the subcritical drift/fluctuation trichotomy.
\Cref{sec:numerical experiments} presents numerical experiments illustrating the predicted phase transitions.
The appendices collect the proofs and extend the framework to correlated tokens and RoPE.

\section{Setup}\label{sec:main model and assumptions}

We recall that a PostLN attention head~\cite{vaswani2017attention,geshkovski2023mathematical} 
with context $x_1, \ldots, x_n$ is parametrized by three $d\times d$ matrices $Q$, $K$,  
and $V$. It implements the map 
$$
x \in \S \mapsto V\sum_{j=1}^n A_j x_j,
$$
where the attention weights are defined by
\begin{equation*}
A_j \;=\; \frac{\exp\bigl(\beta \langle Qx,Kx_j\rangle\bigr)}
{\sum_{l=1}^n \exp\bigl(\beta \langle Qx,Kx_l\rangle\bigr)},
\qquad j=1,\dots,n\,.
\end{equation*}
Setting $q = K^\top Q x / \|K^\top Q x\| \in \S$ and 
$\beta' = \beta\,\|K^\top Q x\|$, we can rewrite
\begin{equation*}
A_j \;=\; \frac{\exp\bigl(\beta' \langle q, x_j\rangle\bigr)}
{\sum_{l=1}^n \exp\bigl(\beta' \langle q, x_l\rangle\bigr)},
\qquad j=1,\dots,n,
\end{equation*}
so that, up to rescaling $\beta$, the attention weights take the 
form~\eqref{eq:attention-weights-intro}.

Throughout the main theorems, we assume a random context $x_1, \ldots, x_n$ sampled i.i.d from some distribution on the sphere.

\begin{assumption}[i.i.d. spherical context]\label{a:iid}
The query $q\in\S$ is fixed.  The context directions $x_1,\ldots,x_n$ are i.i.d. with common law $\rho(x)\,\diff\sigma(x)$ on $\S$, where $\sigma$ is surface measure.  The density $\rho$ is positive at $q$ and is $C^2$.
\end{assumption}

A few remarks regarding \Cref{a:iid} should be noted.

\noindent{\it Regularity of $\rho$, intrinsic dimension, and off-support geometry.}
The $C^2$ regularity is used only for the attention-output expansions in \Cref{sec:attention_outputs}.  The scaling limits of the ordered weights in \Cref{sec:main results} require only continuity of $\rho$ at $q$ and $\rho(q)>0$. If $\rho$ is supported on a smooth $(d'-1)$-dimensional submanifold, with $d'-1<d-1$, through $q$ with positive local density and non-degenerate contact, then our results hold with $d$ replaced by $d'$. If $q$ is not contained in the support of $\rho$, the same analysis applies after subtracting the smallest attainable D2Q from all D2Qs. 

\noindent{\it Correlations and positional encoding.}
Independence is also not essential to our results.  
The proofs of the scaling limits use only two inputs: a one-point lower-tail law for the D2Q $T_i=1-\langle q,x_i\rangle$ and a Poisson, or more generally extremal, limit for close neighbors to the query $q$. 
Under stationary weak dependence with anti-clustering, the same exponent and Poisson softmax limit persist. We also discuss how positional encodings such as RoPE~\cite{su2024roformer} can also be handled with minor technical adjustments in \Cref{app:proof_correlated_critical}.

%%%%%%%%%%%%%%%%%%%%%%%%%%%%%%%%%%%%%%
\section{Scaling limits of attention weights}\label{sec:main results}

In this section, we consider the scaling limits for the ordered attention weights. We start with an informal derivation of the critical scaling $\beta_n ^{\ast}\asymp n^{2/(d-1)}$ (equivalently \(\beta_n ^{\ast} \asymp n^\alpha\) with \(\alpha=2/(d-1)\)).

\subsection{A heuristic explanation for the critical scaling} 

The critical scale \(\beta_n^* \asymp n^{2/(d-1)}\)   arises precisely at the balance point of two opposing effects. As the context length \(n\) grows, more random tokens from the input naturally fall close to the fixed query direction \(q\). At the same time, increasing the rescaling factor \(\beta_n\) narrows the effective softmax window: the factor \(e^{-\beta_n T_i}\) assigns comparable weight only to tokens whose distance-to-query (D2Q) values \(T_i\) differ on the scale \(1/\beta_n\).

For simplicity, first assume that the distribution \(\rho\) in \Cref{a:iid} is the uniform measure on the unit sphere \(\mathbb{S}^{d-1}\). Geometrically, a D2Q window of width \(1/\beta_n\) corresponds to a small ball around \(q\). Indeed, for tokens \(x_i\) sufficiently close to \(q\), we have the local approximation \(T_i \approx \frac12 \dist(q,x_i)^2\). Consequently, the active softmax window has radius of order \(\beta_n^{-1/2}\). Since this region is intrinsically \((d-1)\)-dimensional, the expected number of context tokens falling inside the window scales as
\[
n \cdot \operatorname{Vol}\!\bigl(B_{d-1}(q,\beta_n^{-1/2})\bigr)
\asymp n \beta_n^{-(d-1)/2}
= n \beta_n^{-1/\alpha}.
\]
Here, $\operatorname{Vol}\!\bigl(B_{d-1}(q,\beta_n^{-1/2})\bigr)$ is the volume of the geodesic ball around $q$ of radius $\beta_n^{-1/2}$. 
If \(\beta_n\) grows too rapidly with \(n\), the window shrinks faster than new tokens enter it and the expected count tends to zero: attention collapses onto the single nearest token. Conversely, if \(\beta_n\) grows too slowly, this number diverges and attention is effectively averaged over a diverging number of tokens. The critical regime occurs precisely when these two effects balance, i.e.,
\[
n \beta_n^{-1/\alpha} \asymp 1,
\qquad\text{or equivalently}\qquad
\beta_n \asymp n^\alpha.
\]

The same scaling holds for general distributions \(\rho\) satisfying \Cref{a:iid}. Indeed, \(\mathbb{P}(T_i < t) \asymp t^{1/\alpha}\) (see \Cref{lem:small_cap_asymptotics}), so the expected number of tokens lying in a D2Q window of width \(\beta_n^{-1}\) is
\[
n \mathbb{P}(T_i < \beta_n^{-1}) \asymp n \beta_n^{-1/\alpha}.
\]

The next three results make this intuitive picture precise. In each case, the main technical step is to pass from the tail law \(\mathbb{P}(T_i < t) \asymp t^{1/\alpha}\) to the corresponding limit law for the softmax denominator.

\subsection{Supercritical regime}
When $\beta_n\gg n^\alpha$, the inverse temperature is high enough to distinguish the top D2Q from all its competitors. The attention function therefore becomes a
hard nearest-neighbor selector.

\begin{theorem}[Supercritical regime]\label{thm:intro super_critical}
    Under \Cref{a:iid}, choose $\beta_n$ in \eqref{eq:attention-weights-intro} such that $\beta_n n^{-\alpha}\to\infty$ as $n \to \infty$. We have that for every fixed $i\geq 2$,
    \begin{align}
    \lim_{n\to\infty} A_{(1)} = 1 , \qquad
    \lim_{n\to\infty} A_{(i)} = 0 ,
    \qquad\text{in probability}.
\end{align}
\end{theorem}
The proof of \Cref{thm:intro super_critical} is deferred to \Cref{app:proof sup_critical}.

\subsection{Critical regime}
At the critical scale $\beta_n\asymp n^\alpha$, the expected number of tokens in the active softmax window remains of order one.  
The correct limiting object is the point process of rescaled D2Qs.

\begin{theorem}[Critical regime]\label{thm:intro critical}
Under \Cref{a:iid}, choose $\beta_n$ in \eqref{eq:attention-weights-intro} such that $\beta_n n^{-\alpha} \to \gamma\in(0,\infty)$. Let $\Np$ be a Poisson point process on $\mathbb R_+$ with intensity measure
\(
    \Lambda([0,y])
    = C(q)(y/\gamma)^{1/\alpha}
\)
, for $y\geq0$. Write $0<Y_1<Y_2<\cdots$ for its atoms in increasing order.  Then, for every fixed $k\geq1$,
\begin{align}
    (A_{(1)},\dots,A_{(k)}) \rightarrow (W_{1},\dots,W_{k})  \qquad \text{in distribution,}
\end{align}
where 
\begin{align*}
    W_i = \frac{e^{-Y_i}}{\sum_{j=1}^\infty e^{-Y_{j}}},
\end{align*}
and $\{Y_i\}_{i\geq1}$ are the increasing atoms of the Poisson point process $\Np$, i.e., $Y_i$ is the $i$-th arrival of $\Np$.
\end{theorem}
This is the selective regime: finitely many close neighbors retain nonzero mass, but no deterministic winner is selected.  The proof is given in \Cref{app:proof critical}.

\subsection{Subcritical regime}\label{sec:intro subcritical regime}
When $\beta_n\ll n^\alpha$, the active softmax window contains a diverging number of near-query tokens.  Define the active softmax window size by $m_n(q)\coloneqq C(q)n\beta_n^{-1/\alpha}$.
In the subcritical regime with $\beta_n\to\infty$, we have $m_n(q)\to\infty$.  
Hence no individual token can retain macroscopic
mass.  Nevertheless, after ordering by score, the top \(O(m_n(q))\) weights secure a
deterministic amount of total weight.

\begin{theorem}[Subcritical regime]\label{thm:intro sub_critical 2 new}
Under \Cref{a:iid}, choose $\beta_n$ in \eqref{eq:attention-weights-intro} such that $\beta_n n^{-\alpha}\to 0$ as $n \to \infty$. We have that for every $i\geq 1$,
    \begin{align}
        \lim_{n\to\infty} A_{(i)} = 0 \qquad \text{ in probability} \,.
    \end{align}
If $\beta_n \to \infty$, and let $\{k_n\}_{n \ge 1}$ be a sequence of integers such that 
\(
    \lim_{n \to \infty} \tfrac{k_n}{m_n(q)}  = x \geq 0,
\)
for some nonnegative constant $x$ (for instance $k_n = \lfloor x m_n(q) \rfloor$). We have that
    \begin{equation}
        \lim_{n\to\infty} \frac{A_{(k_n)}}{A_{(1)}} = \exp\left( -x^{\alpha} \right) \qquad \text{ in probability} ,
    \end{equation}
Indeed, we have that
    \begin{equation}
        \lim_{n\to\infty}     m_n(q) A_{(k_n)} =  \frac{\exp\left( - x^{\alpha} \right)}{\Gamma\left(\frac{1}{\alpha}+1\right)} \qquad \text{ in probability} ,
    \end{equation}
where $\Gamma(\cdot)$ is the gamma function. Moreover, 
\begin{equation}\label{eq:subcritical_cumulative_mass_limit}
    \lim_{n\to\infty} \sum_{i=1}^{k_n} A_{(i)}=
    \frac{\gamma_{\rm inc}\!\left(\tfrac1\alpha,x^\alpha\right)}
         {\Gamma\!\left(\tfrac1\alpha\right)} \ , \qquad \text{in probability} \ ,
\end{equation}
where
\(
    \gamma_{\rm inc}(s,z)
    \coloneqq \int_0^z u^{s-1}e^{-u}\,\mathrm{d}u
\)
is the lower incomplete gamma function.
\end{theorem}

The proofs for \Cref{thm:intro sub_critical 2 new} is included in \Cref{app:proof sub_critical}.

\section{Scaling limits of attention outputs}\label{sec:attention_outputs}

\Cref{sec:main results} describes how the ordered attention weights $A_{(1)} \ge A_{(2)} \ge \dots \ge A_{(n)}$ are distributed. We now track the vector returned by the attention layer $\att(n)$.  Under \Cref{a:iid}, we show in \Cref{sec:sup and critical attn output} and \Cref{sec:subcritical attn output} that all three regimes with $\beta_n\to\infty$ concentrate on those tokens approaching $q$, so that the first-order limit of $\att(n)$ is always $Vq$.  We therefore study the  displacement
\begin{align}\label{eq:output_displacement_again}
    \mathcal Y_n(q)\coloneqq \att(n)-Vq
    =V\sum_{i=1}^n A_i(x_i-q).
\end{align}
The two relevant scales in \Cref{sec:sup and critical attn output} and \Cref{sec:subcritical attn output} are $\beta_n\asymp n^\alpha$, the ordered-weight transition, and $\beta_n\asymp n^{\alpha/(1+\alpha)}=n^{2/(d+1)}$, the drift/fluctuation transition inside the subcritical regime.

\subsection{Extreme-value output limits}\label{sec:sup and critical attn output}
In this subsection, we record attention outputs in the supercritical and the critical regimes in which only finitely many near-query tokens are
visible at the output scale. 

First, for the supercritical regime, we prove in the following \Cref{thm:output_supercritical} that the displacement $\mathcal Y_n(q)$ scales as $n^{-\alpha/2}$.
\begin{theorem}[Supercritical output]\label{thm:output_supercritical}
Under \Cref{a:iid}, assume that
\(
    \beta_n n^{-\alpha}\to\infty .
\)
Then for $a_n = (C(q)n)^{\alpha}$, we have
\[
    \lim_{n\to\infty}\sqrt{a_n}\,\mathcal Y_n(q)
    =
    V\Phi  \qquad \text{ in distribution} ,
\]
where the random vector $\Phi=\sqrt{2R_1}\,\Theta_1$, \(R_1\) is a Weibull random variable with
\[
    \mathbb P(R_1>r)=e^{-r^{1/\alpha}},\qquad r\ge0,
\]
and \(\Theta_1\) is uniformly distributed over 
\(
\mathbb S_q^{d-2}\coloneqq \{\theta\in T_q\mathbb S^{d-1}:\|\theta\|=1\}
\)
, and is independent of $R_1$.
In particular, \(\mathcal Y_n(q)\to0\) in probability, and hence \(\att(n)\to Vq\) in probability.
\end{theorem}

\Cref{thm:output_supercritical} says that, once the weights have collapsed to a single winner, the
only remaining randomness is the nearest-neighbor geometry: its radial D2Q has
a Weibull-type extreme-value law, while its angular mark is asymptotically uniform
in the tangent space.

We next show the attention outputs within the critical regime.
\begin{theorem}[Critical output]\label{thm:output_critical}
Under \Cref{a:iid}, assume that
\(
    \beta_n n^{-\alpha}\to \gamma\in(0,\infty).
\)
There exists a Poisson point process \(\mathcal M_\gamma\) on \([0,\infty)\times\mathbb S_q^{d-2}\) with intensity measure depending on $\gamma$ and $\rho(q)$ in Assumption \ref{a:iid}, such that as $n\to\infty$
\[
    \lim_{n\to\infty}\sqrt{\beta_n}\,\mathcal Y_n(q)
    =
    V\Xi  \qquad \text{ in distribution} .
\]
where
\begin{align}\label{eq:critical_output_functional_main}
    \Xi
    \coloneqq
    \frac{\displaystyle\int_{0} ^{\infty} \int_{\mathbb S_q^{d-2}} e^{-y}\sqrt{2y}\,\theta\,\mathcal M_\gamma(dy,d\theta)}
         {\displaystyle\int_{0} ^{\infty} \int_{\mathbb S_q^{d-2}} e^{-y}\,\mathcal M_\gamma(dy,d\theta)}.
\end{align}
In particular, \(\mathcal Y_n(q)\to0\) in probability, and hence \(\att(n)\to Vq\) in probability.
\end{theorem}
The marked Poisson process in \Cref{thm:output_critical} is the output analogue of
the Poisson process governing the critical ordered weights in \Cref{thm:intro critical}.  Its scalar coordinate
records how far a token is from the optimal score, while the mark records the
local direction in which that token perturbs the output.  The denominator in
\eqref{eq:critical_output_functional_main} is the limiting softmax partition
function; the numerator is the corresponding tangent displacement.

\subsection{Subcritical local averaging}\label{sec:subcritical attn output}

When \(\beta_n\ll n^\alpha\), the number of relevant near-query tokens diverges.
The output is then governed by local averaging rather than by a finite extreme
process.  Define
\begin{equation}\label{eq:subcritical_drift_cov_main}
    \mathcal Y(q)
    \coloneqq
    \nabla_{\S}\log\rho(q)-\frac{d-1}{2}q,
    \qquad
    \Sigma(q)
    \coloneqq \frac{c_d}{\rho(q)}\proj_q,
    \qquad
    c_d\coloneqq 2^{-d}\pi^{-(d-1)/2},
\end{equation}
where \(\proj_q=\Id-qq^\top\) is the orthogonal projection onto the tangent space \(T_q\S\).  

\begin{theorem}[Subcritical output]
\label{thm:intro sub_critical 3}
Under \Cref{a:iid}, choose $\beta_n$ in \eqref{eq:attention-weights-intro} such that $\beta_n \to \infty$ and $\beta_n n^{-\alpha}\to 0$ as $n \to \infty$. 
\begin{enumerate}
    \item \textbf{Drift dominated.} If
    $\beta_n^{1 + \alpha} n^{-\alpha}\to 0$,
    then
    \begin{align*}
        \lim_{n \to \infty} \beta_n \mathcal{Y}_n(q) = V \mathcal{Y}(q) \qquad \text{ in probability} \,.
    \end{align*}

    \item \textbf{Mixed drift--fluctuation regime.} If
    $\beta_n^{1 + \alpha} n^{-\alpha}\to\tau\in(0,\infty)$,
    then
    \begin{align*}
        \lim_{n \to \infty} \beta_n \mathcal{Y}_n(q) = V \mathcal{Y}(q)+G_\tau \qquad \text{ in distribution} \,.
    \end{align*}
    where $G_\tau\sim\mathcal N(0,\tau^{1/\alpha} V\Sigma(q)V^{\top})$.
    \item \textbf{Fluctuation dominated.}
    If
    $\beta_n^{1 + \alpha} n^{-\alpha}\to \infty$,
    then
    \begin{align*}
        \lim_{n \to \infty} \sqrt{n\beta_n^{-1/\alpha + 1}}\, \mathcal{Y}_n(q) = G_1 \qquad \text{ in distribution} \,.
    \end{align*}
    where $G_1\sim\mathcal N(0,V\Sigma(q) V^{\top})$
\end{enumerate}
In all three regimes, we have that $\mathcal{Y}_n(q) \to 0$ in probability, and hence $\att(n)\to Vq$ in probability.
\end{theorem}
The condition \(\beta_n^{1+\alpha}n^{-\alpha}\asymp1\) is equivalent to \(\beta_n\asymp n^{\alpha/(1+\alpha)}=n^{2/(d+1)}\).  Thus the subcritical regime contains its own output-level phase transition: for $\beta_n$ at or below this critical scale, the deterministic local-density drift is present, while it is swamped by Gaussian fluctuations above it.
    
\subsection{Residual dynamics}

We now study how one attention layer changes the relative geometry of the tokens
in the drift-dominated subcritical regime, tracking
the cosine similarity between normalized tokens before and after one layer.  For
\(V=\Id\), the attention output can be written as
\[
    \att(n)=q+\mathcal Y_n(q),
\]
where, by \Cref{thm:intro sub_critical 3}(1),
\(\mathcal Y_n(q)\) is of order \(\beta_n^{-1}\).
Therefore the layer-normalized attention output can be approximated as
\[
    q'
    =
    \frac{\att(n)}{\|\att(n)\|}
    =
    \frac{q+\mathcal Y_n(q)}{\|q+\mathcal Y_n(q)\|}
    \simeq
    q+\frac{1}{\beta_n}\nabla_{\S}\log\rho(q).
\]
The proposition below formalizes this observation and makes this claim rigorous.

\begin{proposition}\label{prop:backward_heat}
Assume $V=\Id$. Let $q \in\mathbb S^{d-1}$ satisfy $\rho(q)>0$, and define $q'$ as 
\begin{align}
\label{eq:backward_heat}
    q' \coloneqq \frac{q+\gamma\mathcal Y_n(q)}{\|q+\gamma\mathcal Y_n(q)\|},
\end{align}
for some fixed $\gamma>0$. In the drift-dominated regime where $\beta_n^{1+\alpha}n^{-\alpha}\to0$ and $\beta_n\to\infty$, it holds
\begin{align}
\label{eq:heatapprox}
    \beta_n (q ' - q) \to \gamma \nabla_{\mathbb S^{d-1}}\log\rho(q) ,
\end{align}
in probability.
\end{proposition}
The above proposition indicates that attention, together with PostLN, approximately implements the update
$$
q \mapsto q + \frac{\gamma}{\beta_n} \nabla_{\mathbb S^{d-1}} \log \rho(q),
$$
which is precisely a one-step Euler discretization of $\dot q =  \nabla_{\mathbb S^{d-1}} \log \rho(q)$. In turn, it is a celebrated result of Jordan, Kinderlehrer, and Otto~\cite{jordan1998variational} that this Wasserstein gradient flow of the entropy functional~\cite{CheNilRig25} implements the reverse heat equation.

The proof of \Cref{prop:backward_heat} is postponed to Appendix~\ref{app:attention_outputs}.

%%%%%%%%%%%%%%%%%%%%%%%%%%%%%%%%%%%%%%%
\section{Numerical Experiments}\label{sec:numerical experiments}
The following experiments illustrate three consequences of the scaling theory from the main results on the model introduced in
\eqref{eq:attention-weights-intro} and \eqref{eq:output_displacement_again}.

\textbf{Output field on the sphere.}
We first visualize the output vector fields from
\Cref{thm:output_supercritical,thm:output_critical,thm:intro sub_critical 3}.
We take $d=3$, so $\alpha=1$, set $V=\Id$, and draw $n=10^4$ context tokens from
\[
    \rho(x)\propto \exp(x_1x_2),\qquad x\in\mathbb S^2.
\]
For each query direction $q$ we compute
\(
    \mathcal Y_n(q)=\sum_i A_i(x_i-q)
\)
and plot the first coordinate of $\proj_q\mathcal Y_n(q)$ after the
corresponding scaling.  With this choice of parameters the critical scaling is
\(\beta_n\asymp n\), while the subcritical drift/fluctuation scaling from
\Cref{thm:intro sub_critical 3} is
\(\beta_n\asymp n^{1/2}\).  Thus the columns of
\Cref{fig:output_field_regimes} show the supercritical regime
(\(\beta_n=n^{5/4}\)), the critical regime (\(\beta_n=n\)), the
fluctuation-dominated subcritical regime (\(\beta_n=n^{3/4}\)), the mixed
drift--fluctuation regime (\(\beta_n=n^{1/2}\)), and the drift-dominated
subcritical regime (\(\beta_n=n^{1/4}\)).

The first two empirical columns are visibly rough: in the supercritical regime the
output is essentially determined by one nearest token, while at criticality only a
finite number of close neighbors contributes.
Moving into the subcritical regime, more tokens enter the local average
and the field becomes smoother.  The
fluctuation-dominated panel still shows sampling noise; the mixed panel begins to
show the large-scale drift; and the drift-dominated panel has the same structure
as the deterministic density-gradient field \(\proj_q\mathcal Y(q)\) from
\eqref{eq:subcritical_drift_cov_main}, shown in the last column.

\begin{figure}[hbt!]
    \centering
\includegraphics[width=0.95\linewidth]{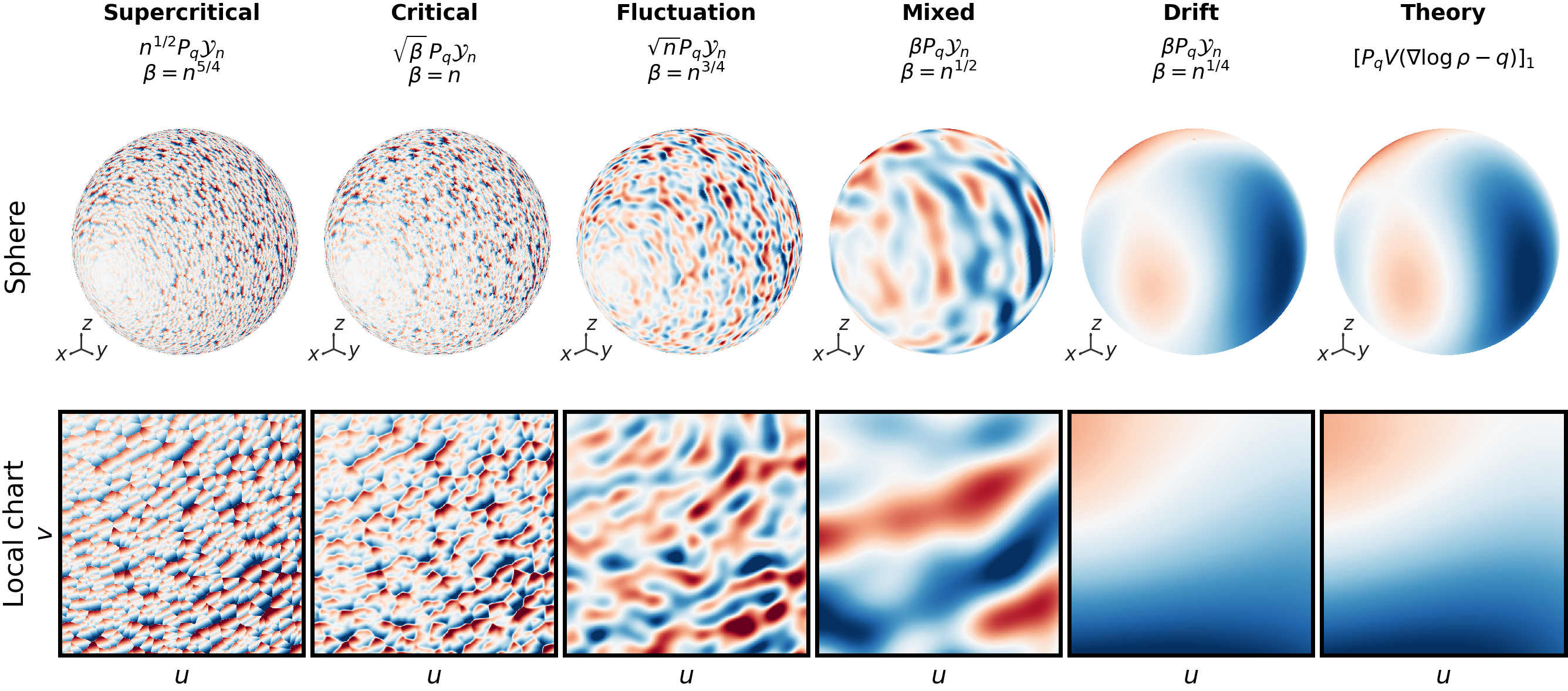}
    \caption{Rescaled output displacement on $\mathbb S^2$ for $n=10^4$, $V=\Id$, and
    $\rho(x)\propto\exp(x_1x_2)$.  The columns use $\beta_n=n^{5/4},n,n^{3/4},n^{1/2},n^{1/4}$, with the regime-dependent
    scalings shown above the panels, and the last column is the deterministic
    drift field.  The bottom row is a geodesic chart centered at
    $(1,1,1)/\sqrt3$.}
    \label{fig:output_field_regimes}
\end{figure}

\textbf{Ordered weights.} We next check the ordered attention weight statements in
\Cref{thm:intro super_critical,thm:intro critical,thm:intro sub_critical 2 new}.
Contexts are uniform on $\mathbb S^4$, so $d=5$ and $\alpha=1/2$.
The left panel of \Cref{fig:weight_numerics} estimates \(A_{(1)}\) over a
logarithmic grid in \((n,\beta_n)\), averaging over \(100\) trials.  In agreement
with \Cref{thm:intro super_critical,thm:intro sub_critical 2 new}, the largest
weight is small below the critical scaling and approaches one above it; the
transition occurs at \(\beta_n\asymp n^\alpha\).
The right panel fixes \(n=1000\), \(q=e_1\), and
\(\beta_n=n^{\alpha/4}\).  It plots the ratios \(A_{(k)}/A_{(1)}\) against the
rescaled rank \(x=k/m_n(q)\), and
compares them with the prediction \(e^{-x^\alpha}\) in
\Cref{thm:intro sub_critical 2 new}.

\begin{figure}[hbt!]
    \centering
    \begin{subfigure}[t]{0.4\textwidth}
        \centering
        \includegraphics[width=\textwidth]{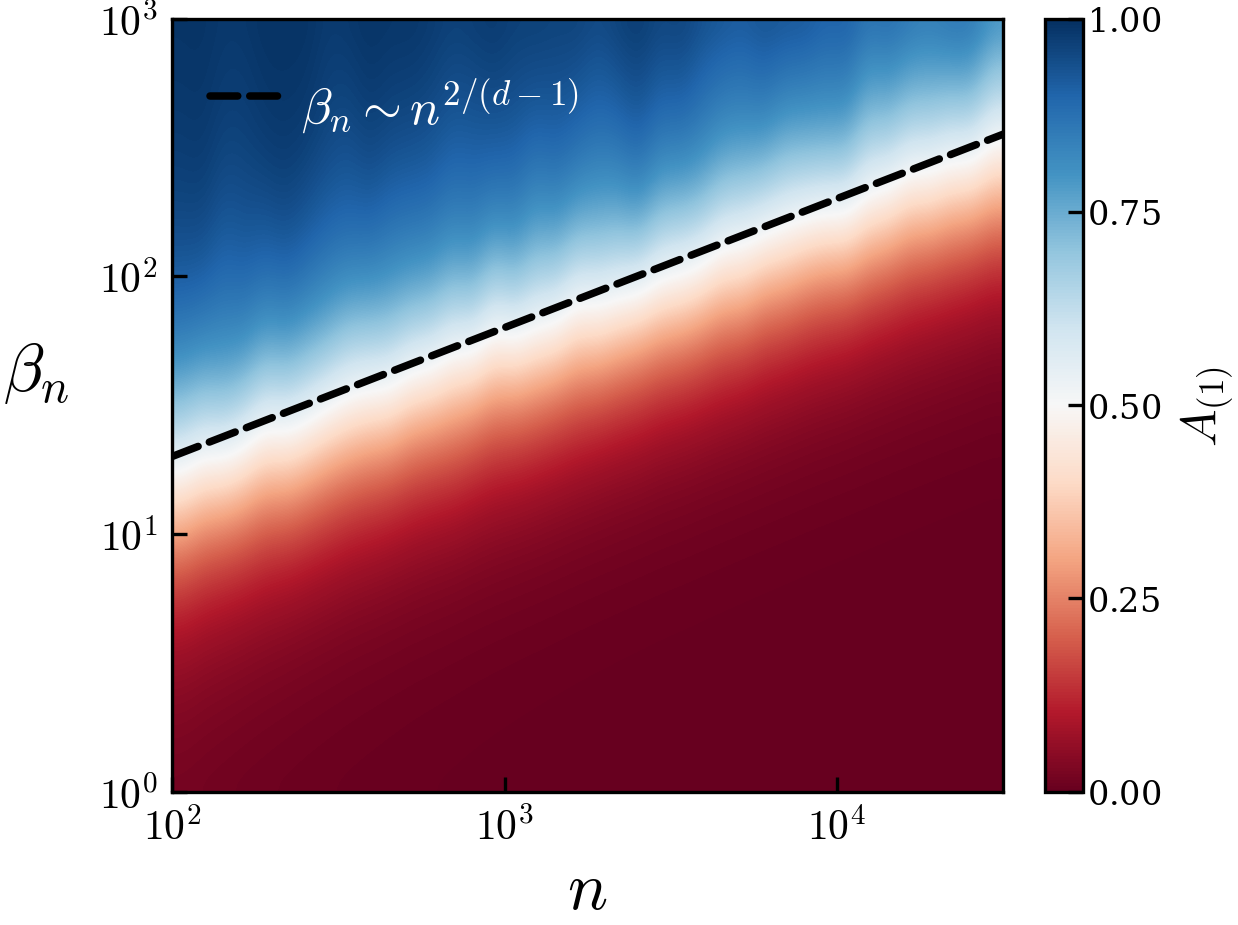}
    \end{subfigure}
    \hfill
    \begin{subfigure}[t]{0.4\textwidth}
        \centering
        \includegraphics[width=\textwidth]{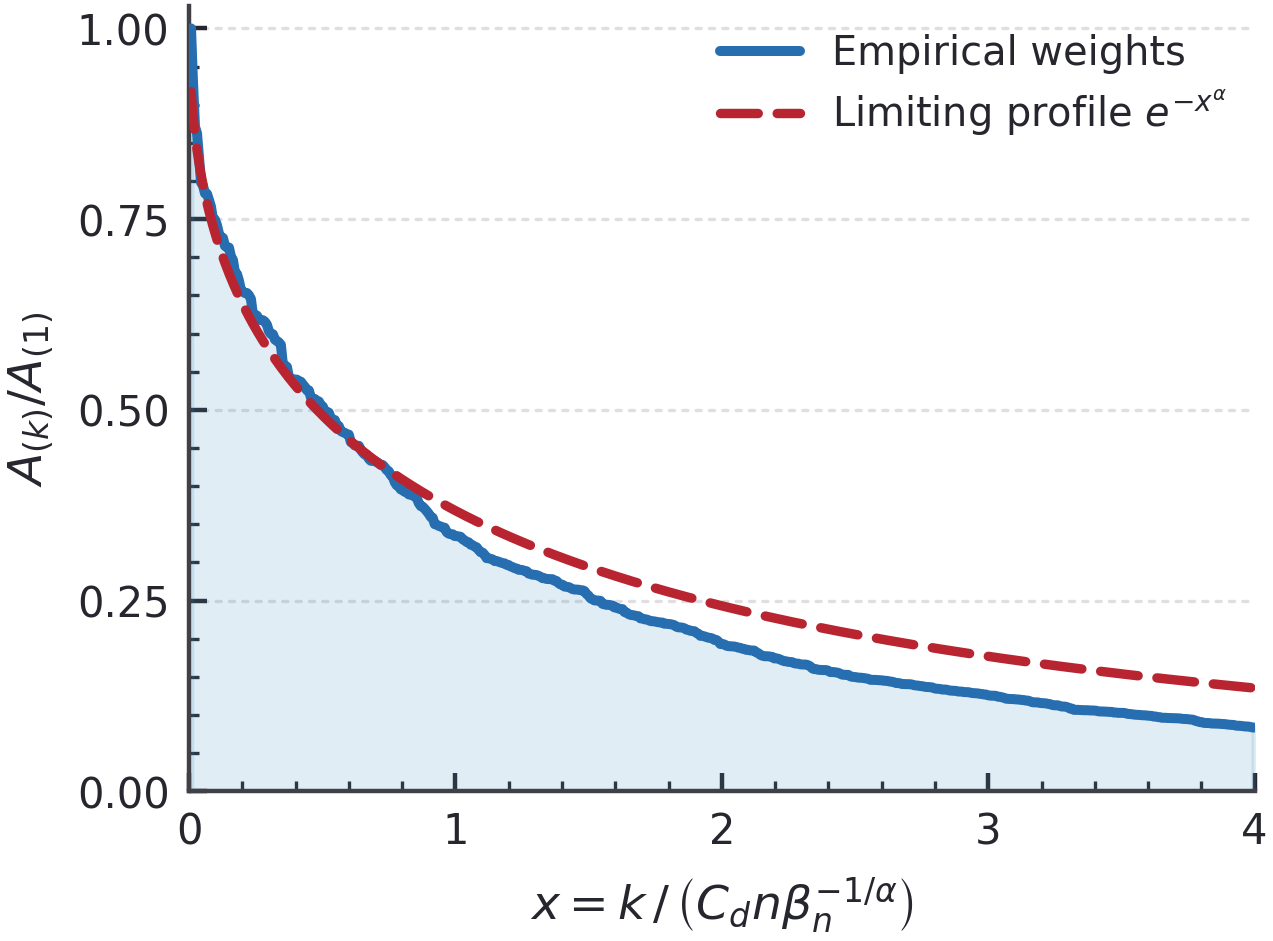}
    \end{subfigure}
    \caption{Ordered attention weights for uniform contexts on $\mathbb S^4$.  Left: the heatmap shows the empirical mean of \(A_{(1)}\)
    over \(100\) trials; the dashed curve \(\beta_n=2n^\alpha\) marks the
    critical scaling.  Right: with \(\beta_n=n^{\alpha/4}\), the rescaled ordered
    weights follow the subcritical prediction \(e^{-x^\alpha}\), where
    \(x=k/m_n(q)\).}
    \label{fig:weight_numerics}
\end{figure}

%------------------------------------------------
%------------------------------------

\section{Conclusion}\label{sec:conclusion}

We introduced a tractable probabilistic model to better understand attention scaling in the long-context regime. Within this framework we obtained a complete classification of the scaling limits for both the ordered attention weights and the resulting attention outputs. The central result is that the optimal rescaling takes the form $\beta_n \asymp n^{2/(d-1)}$, governed by the local extreme-value tail of the distance-to-query distribution $\mathbb{P}(T_i < t)$. To the best of our knowledge, this is the first scaling prescription for attention that exhibits an explicit dependence on the embedding dimension~$d$, reflecting the intrinsic geometry of the key space. Below this threshold every attention weight vanishes; above it the mechanism collapses to a single-winner selector; only at the critical scale do the ordered weights converge to a non-degenerate Poisson point process.

%------------------------------------

%-------------------------------------

\section*{Acknowledgment}

PR was partially supported by NSF grant DMS-2509011. GB thanks Andrea Agazzi for their continued support throughout this work.

\newpage

\begin{appendix}
    
\phantomsection
\addcontentsline{toc}{section}{Appendices}
\addtocontents{toc}{\protect\setcounter{tocdepth}{-1}}

% \begin{appendix}

\section{Introduction to the Appendix, Notation, and Shared Lemmas}\label{app:appendix_preliminaries}

The appendices are organized as follows. Appendix~\ref{app:proof sup_critical} proves the supercritical single-winner regime. Appendix~\ref{app:proof critical} proves the critical Poisson point process limit and the induced limiting ordered weights. Appendix~\ref{app:proof sub_critical} proves the subcritical results in the same order as the main text: scalar softmax estimates and ordered weights first, then the local moment estimates and central limit argument for the attention output, and finally the backward heat-type behavior of long-context attention. Appendix~\ref{app:proof_correlated_critical} treats the RoPE extension with correlated tokens.

We use the following notation throughout the appendices. The query $q\in\S$ is fixed, $x_1,x_2,\ldots$ are drawn according to the density $\rho$ in \Cref{a:iid}. The D2Qs and their CDF are
\[
    T_i\coloneqq 1-\langle q,x_i\rangle\in[0,2],
    \qquad
    F(t)\coloneqq \mathbb P(T_i\le t).
\]
We write
\[
    T_{(1)}\le T_{(2)}\le\cdots\le T_{(n)}
\]
for the order statistics of the D2Qs and
\[
    A_i=\frac{e^{-\beta_nT_i}}{Z_n},
    \qquad
    Z_n\coloneqq \sum_{j=1}^n e^{-\beta_nT_j},
    \qquad
    A_{(1)}\ge A_{(2)}\ge\cdots\ge A_{(n)}
\]
for the corresponding attention weights and their decreasing rearrangement. We also set
\[
    \alpha\coloneqq \frac{2}{d-1},
    \qquad
    \eta\coloneqq \frac{1}{\alpha}-1=\frac{d-3}{2},
    \qquad
    \proj_q\coloneqq \Id-qq^\top.
\]
We use $\Theta(h)$ to denote a quantity comparable to $h$, i.e., there are constants $C_1$, $C_2$ depending on $\rho$ as introduced in \Cref{a:iid} such that $C_1 h \leq \Theta(h) \leq C_2 h$. Similarly, we use $x \approx y$ if there are two positive constants $C_1,C_2$, which depend at most on $d$ and $\rho$ in \Cref{a:iid}, such that $C_1 x \leq y \leq C_2 x$.
The symbol $o_{\mathbb P}(1)$ denotes a random sequence converging to $0$ in probability. More generally, $R_n=o_{\mathbb P}(a_n)$ means $R_n/a_n\to0$ in probability, and $R_n=O_{\mathbb P}(a_n)$ means that $R_n/a_n$ is tight.

The CDF $F(t)$ is of order $\Theta(t^{1/\alpha})$, because
\[
\begin{aligned}
F(t) \coloneqq \mathbb{P}\left(T_i<t\right) = \mathbb{P}\left(\dist(q,x)<2 \arcsin\left(t^{1/2} 2^{-1/2}\right)\right) = \Theta\left(t^{\tfrac{d-1}{2}}\right) = \Theta\left(t^{1/\alpha}\right).
\end{aligned}
\]
Here, $\dist(q,x)$ is the geodesic distance between $q$ and $x$ on the unit sphere $\S$, and hence exactly equals the angle between $q$ and $x$.
In particular, we denote
\[
C(q) \coloneqq \lim_{t\to 0^+} F(t) t^{-1/\alpha} \,.
\]
We summarize the asymptotics of $F(t)$ in the following two lemmas.
\begin{lemma}[CDF asymptotics]\label{lem:small_cap_asymptotics}
Under \Cref{a:iid}, the law of $T_i$ is absolutely continuous on $[0,2]$, and has a density $f_T(t) \coloneqq F'(t)$. Moreover, there are constants $0<c_-<c_+<\infty$ such that
\[
    c_-t^{1/\alpha}\le F(t)\le c_+t^{1/\alpha}, \quad c_-t^{\eta}\le f_T(t)\le c_+t^{\eta}
    \qquad \forall t \in [0,2].
\]
In particular, as $t\to0^+$,
\[
    F(t)=C(q)t^{1/\alpha}+o\bigl(t^{1/\alpha}\bigr),
    \text{ with }C(q) = \frac{2^{1/\alpha}\sigma_{d-2}}{d-1}\rho(q),
\]
where $\sigma_{d-2}$ denotes the surface area of $\mathbb S^{d-2}$.
\end{lemma}

\begin{lemma}[Quantile asymptotics]\label{lem:inverse_cdf}
As $u\to0^+$,
\[
    F^{-1}(u)=\left(\frac{u}{C(q)}\right)^\alpha(1+o(1)).
\]
The same relation holds with $u$ replaced by any random sequence $U_n\to0$ in probability: 
\[
    F^{-1}(U_n)=\left(\frac{U_n}{C(q)}\right)^\alpha(1+o_{\mathbb P}(1)).
\]
\end{lemma}

We use the following form of R\'enyi's representation for uniform order
statistics; see \cite[Chapter~5, Theorem~2.2]{devroye2006nonuniform}.
\begin{lemma}[R{\'e}nyi's Representation]\label{lem:Renyi representation}
Let \( U_1, U_2, \dots, U_n \) be i.i.d. \(\operatorname{Uniform}(0,1)\) random variables, and let \( U_{(1)} < U_{(2)} < \dots < U_{(n)} \) denote the associated order statistics.

Let \( E_1, E_2, \dots, E_{n+1} \) be i.i.d. exponential random variables with rate 1, i.e., \( E_i \sim \operatorname{Exp}(1) \). Define the partial sums
\[
\Gamma_k = \sum_{i=1}^k E_i, \quad k=1,2,\dots,n+1.
\]

Then
\[
\bigl( U_{(1)}, U_{(2)}, \dots, U_{(n)} \bigr) 
\stackrel{d}{=}
\left( \frac{\Gamma_1}{\Gamma_{n+1}}, \frac{\Gamma_2}{\Gamma_{n+1}}, \dots, \frac{\Gamma_n}{\Gamma_{n+1}} \right).
\]
\end{lemma}

We will repeatedly use the standard point-process formulation. Let \(\Omega\) be a locally compact Polish space. We denote by \(M_p(\Omega)\) the space of boundedly finite counting measures on \(\Omega\), equipped with the vague topology; see \cite{daley2008introduction}. A point process on \(\Omega\) can be viewed as an \(M_p(\Omega)\)-valued random element. Thus, when
\[
    \Np_n \Rightarrow \Np \qquad \text{in } M_p(\Omega),
\]
we mean weak convergence of the laws of the random counting measures
\(\Np_n\) to the law of \(\Np\) in this topology.

We will use the Laplace-functional characterization of convergence of point processes. Namely, to prove \(\Np_n \Rightarrow \Np\) in \(M_p(\Omega)\), it is enough to show that, for every \(\varphi\in C_c^+(\Omega)\),
\[
\mathbb E\left[
    \exp\left\{
        -\int_{\Omega}\varphi(y)\,\Np_n(dy)
    \right\}
\right]
\longrightarrow
\mathbb E\left[
    \exp\left\{
        -\int_{\Omega}\varphi(y)\,\Np(dy)
    \right\}
\right],
\]
where \(C_c^+(\Omega)\) denotes the set of nonnegative continuous functions on \(\Omega\) with compact support; see, e.g.,
\cite[Chapter~11]{daley2008introduction}.

\section{Supercritical Regime: Proof of \Cref{thm:intro super_critical}}\label{app:proof sup_critical}

We use the notation from Appendix~\ref{app:appendix_preliminaries}. In particular,
\[
    A_i=\frac{e^{-\beta_nT_i}}{Z_n},
    \qquad
    Z_n=\sum_{j=1}^n e^{-\beta_nT_j},
    \qquad
    T_{(1)}\le\cdots\le T_{(n)}.
\]
By \Cref{lem:small_cap_asymptotics},
\[
    F(t)=C(q)t^{1/\alpha}+o\bigl(t^{1/\alpha}\bigr),
    \qquad
    C(q)=\frac{2^{1/\alpha}\sigma_{d-2}}{d-1}\rho(q),
\]
and the density $f(t)=F'(t)$ satisfies $f(t)=\Theta(t^\eta)$ near $0$, where $\eta=1/\alpha-1$.
The largest attention weight can be written as
\begin{align*}
    A_{(1)}=\frac{e^{-\beta_nT_{(1)}}}{\sum_{k=1}^n e^{-\beta_nT_k}}
    =\frac{1}{1+S_n},
    \qquad
    S_n\coloneqq\sum_{k=2}^n e^{-\beta_n(T_{(k)}-T_{(1)})}.
\end{align*}

We first prove the following lemma for $S_n$. 
\begin{lemma}\label{lem:S_n asymptotics}
    For any $n \to \infty$ and any $\beta_n \to \infty$, we have that
        \begin{align}
            \mathbb{E}[S_n]  \approx n^2 \left( O\left( e^{-\beta_n}\right) + O\left( e^{-C'n}\right)+ \beta_n^{-1/\alpha} \mathbf{B}\left(1,n-1\right) + \beta_n^{- 1} \mathbf{B}\left(2-\alpha,n-1\right)\right).
        \end{align}
    Here, $C'$ is a positive constant depending on $\rho$ in \Cref{a:iid}, and $\mathbf{B}(\cdot,\cdot)$ is the beta function.
\end{lemma}

We defer the proof of \Cref{lem:S_n asymptotics} after proving \Cref{thm:intro super_critical}.

\begin{proof}[Proof of \Cref{thm:intro super_critical}]
We show that if $\beta_n n^{-\alpha} \to \infty$, we have that
\begin{align}
    \lim_{n\to \infty} \mathbb{E}[S_n] =0 ,
\end{align}
and thus $S_n \to 0$ in probability.

By \Cref{lem:S_n asymptotics} and the classical asymptotics of the beta function \cite{olver1997asymptotics}, we have that as $n \to \infty$,
    \begin{align}
        \mathbf{B}\left(1,n-1\right) \approx n^{-1}, \text{ and } \mathbf{B}\left(2-\alpha,n-1\right) \approx n^{-2+\alpha}.
    \end{align}
Hence, we have that
    \begin{align}
        n^{2} \beta_n ^{-1/\alpha} \mathbf{B}\left(1,n-1\right) \approx \beta_n ^{-1/\alpha} n, \text{ and } n^{2} \beta_n ^{-1} \mathbf{B}\left(2-\alpha,n-1\right) \approx \beta_n ^{-1} n^{\alpha}, 
    \end{align}
both of which go to $0$ if $\beta_n n^{-\alpha} \to \infty$. Because $O\left( e^{-\beta_n}\right) + O\left( e^{-C'n}\right)$ also goes to $0$, we thus get that $\lim_{n\to \infty} \mathbb{E}[S_n] =0$.
\end{proof}

\begin{proof}[Proof of \Cref{lem:S_n asymptotics}]
Recall that $\eta = \frac{d-3}{2} =\frac{1-\alpha}{\alpha}$. Hence, by the arguments at the beginning of \Cref{app:proof sup_critical}, we see that $f(t) = \Theta(t^{\eta})$.

Let $K_n\coloneqq\{ i\in\llbracket 1 ,n \rrbracket \mid T_i = T_{(1)}\}$ and $R_n \coloneqq \#K_n$. First, we see that
\begin{align}
    \sum_{i=1}^n \bone_{i\in K_n} \sum_{j\neq i} e^{-\beta_n(T_j - T_i)} = R_n S_n \,.
\end{align}
Because $T_i$'s are i.i.d. and have continuous density, we have that for any $i \neq j$, $\mathbb{P}\left( T_i = T_j \right) = 0$. Hence, $\mathbb{P}(R_n >1) = 0$ and thus $R_n = 1$ almost surely. We have that
\begin{align}
    \mathbb{E}[S_n] = \mathbb{E}[R_n S_n] = \mathbb{E} \left[\sum_{i=1}^n \bone_{i\in K_n} \sum_{j\neq i} e^{-\beta_n(T_j - T_i)} \right] = \sum_{i=1}^n \mathbb{E} \left[ \bone_{i\in K_n} \sum_{j\neq i} e^{-\beta_n(T_j - T_i)} \right]  \,.
\end{align}
Because $T_i$'s are i.i.d., we have that
\begin{align}\label{eq:S_n ineql 1}
    \begin{split}
        \mathbb{E}[S_n] &= n \mathbb{E}\left[ \bone_{1\in K_n} \sum_{j\geq 2} e^{-\beta_n(T_j-T_1)} \right] = n \sum_{j\geq 2}\mathbb{E}\left[ \bone_{1\in K_n}  e^{-\beta_n(T_j-T_1)} \right]
        \\  &= n(n-1) \mathbb{E}\left[ \bone_{1\in K_n}  e^{-\beta_n(T_2-T_1)} \right] 
        \\  &= n(n-1) \int_0 ^2 f(t) \left(\int_{t} ^2 f(s) \de s\right)^{n-2} \left(\int_t ^2 e^{-\beta_n(s-t)} f(s)\de s\right) \de t,
    \end{split}
\end{align}
where in the last equality, we used the fact that 
    \begin{align*}
        1\in K_n \iff T_1 \leq T_j \text{ for all } j \geq 2.
    \end{align*}
We first estimate the term $ f(t) \left(\int_{t} ^2 f(s) \de s\right)^{n-2} $ for $t \in [0,2]$. We know that $C_1 t^{\eta +1} \leq F(t) \leq C_2 t^{\eta +1}$  for some $C_1,C_2$ depending on $\rho$ in \Cref{a:iid}. We pick a $\delta<1$ depending on $\rho,d$, such that when $t <\delta$, $C_2 t^{\eta +1} < 1$.
Because $\int_\delta ^2 f(s) \de s = F(2) - F(\delta) = 1-F(\delta) < 1$, we have that $\left(\int_{\delta} ^2 f(s) \de s\right)^{n-2} \approx e^{-Cn}$ for some positive $C$ depending on $\rho$ in \Cref{a:iid}. Hence, we see that 
    \begin{align}\label{eq:S_n ineql 3}
        \begin{split}
            & f(t) \left(\int_{t} ^2 f(s) \de s\right)^{n-2} \approx e^{-Cn} +  \bone_{t < \delta} f(t) \left(\int_{t} ^2 f(s) \de s\right)^{n-2} 
            \\  & = e^{-Cn} +  \bone_{t < \delta} f(t) \left(1-F(t)\right)^{n-2} \approx e^{-Cn} +  \bone_{t < \delta} t^{\eta} \left(1-C_1t^{\eta+1}\right)^{n-2}.
        \end{split}
    \end{align}
We next estimate the term $\left(\int_t ^2 e^{-\beta_n(s-t)} f(s)\de s\right)$. By \eqref{eq:S_n ineql 3}, we only need to discuss the case when $t \in [0,\delta]$ as we can use the trivial bound $1$ when $t > \delta$. Because $f(s) = \Theta(s^{\eta})$, we have that
    \begin{align}\label{eq:S_n ineql 2}
        \begin{split}
            &\int_t ^2 e^{-\beta_n(s-t)} f(s)\de s \approx \int_t ^2 e^{-\beta_n(s-t)} s^{\eta} \de s = \int_0 ^{2-t} e^{-\beta_n s} (s+t)^{\eta} \de s
            \\  &\approx \int_0 ^{2-t} e^{-\beta_n s} (s^{\eta}+t^{\eta}) \de s \approx  O\left( e^{-\beta_n}\right)+ \int_0 ^{\infty} e^{-\beta_n s} (s^{\eta}+t^{\eta}) \de s 
            \\  &\approx O\left( e^{-\beta_n}\right) + \beta_n^{-\eta - 1} + \beta_n^{-1} t^{\eta}.
        \end{split}
    \end{align}
Hence, combining \eqref{eq:S_n ineql 3} and \eqref{eq:S_n ineql 2}, we see that $\mathbb{E}[S_n]$ in \eqref{eq:S_n ineql 1} can be estimated as
    \begin{align}
        \mathbb{E}[S_n] &\approx n^2 \left(e^{-Cn} + O\left( e^{-\beta_n}\right) +\int_0 ^{\delta} \left( \beta_n^{-\eta - 1} + \beta_n^{-1} t^{\eta} \right) t^{\eta} \left(1-C_1t^{\eta+1}\right)^{n-2} \de t \right).
    \end{align}
In the above integral, we do a change of variable for $t = s C_1^{-\frac{1}{\eta+1}}$ to $s$ and by the choice of $\delta$, the integral domain for $s$ is $[0,\delta']$ for some $\delta' < 1$. Hence,
    \begin{align}
        \begin{split}
            &\mathbb{E}[S_n] \approx n^2 \left(e^{-Cn} + O\left( e^{-\beta_n}\right) +\int_0 ^{\delta '} \left( \beta_n^{-\eta - 1} + \beta_n^{-1} s^{\eta} \right) s^{\eta} \left(1-s^{\eta+1}\right)^{n-2} \de s \right)
            \\  & \approx n^2 \left(e^{-Cn} + O\left( e^{-\beta_n}\right) + O\left( e^{-C'n}\right)+\int_0 ^{1} \left( \beta_n^{-\eta - 1} + \beta_n^{-1} s^{\eta} \right) s^{\eta} \left(1-s^{\eta+1}\right)^{n-2} \de s \right).
        \end{split}
    \end{align}
The above estimate follows because when $s \in [\delta',1]$, the term $\left(1-s^{\eta+1}\right)^{n-2}$ is an $O( e^{-C'n})$ term for some $C'$ depending on $\rho$. Finally, we make a change of variable $s = r^{\frac{1}{\eta+1}}$, and we combine the two terms $e^{-Cn}$ and $O( e^{-C'n})$, we then have that
    \begin{align}
        \begin{split}
            &\mathbb{E}[S_n]  \approx n^2 \left( O\left( e^{-\beta_n}\right) + O\left( e^{-C'n}\right)+\int_0 ^{1} \left( \beta_n^{-\eta - 1} + \beta_n^{-1} r^{\frac{\eta}{\eta+1}} \right) \left(1-r\right)^{n-2} \de r \right)
            \\  & = n^2 \left( O\left( e^{-\beta_n}\right) + O\left( e^{-C'n}\right)+ \beta_n^{-\eta - 1} \mathbf{B}\left(1,n-1\right) + \beta_n^{- 1} \mathbf{B}\left(\frac{2\eta+1}{\eta+1},n-1\right)\right).
        \end{split}
    \end{align}
Because $\eta = \tfrac{1}{\alpha} - 1$, we finish the proof of \Cref{lem:S_n asymptotics}.

\end{proof}

%%%%%%%%%%%%%%%%%%%%%%%%%%%%%%
\section{Critical Regime: Proof of \Cref{thm:intro critical}}\label{app:proof critical}

We prove \Cref{thm:intro critical} in this section. 
Let \(U_i=F(T_i)\). Then \(U_i\) are i.i.d. uniform on \((0,1)\), and
\(T_{(i)}=F^{-1}(U_{(i)})\). By \Cref{lem:Renyi representation}, we may write
\[
(U_{(1)},\ldots,U_{(n)})
\stackrel{d}{=}
\left(\frac{\Gamma_1}{\Gamma_{n+1}},\ldots,
      \frac{\Gamma_n}{\Gamma_{n+1}}\right),
\]
where \(\Gamma_i=E_1+\cdots+E_i\) and \(E_i\) are i.i.d. exponential random
variables of mean one. Hence
\[
\beta_nT_{(i)}
\stackrel{d}{=}
\beta_nF^{-1}\left(\frac{\Gamma_i}{\Gamma_{n+1}}\right).
\]

Recall from \Cref{lem:inverse_cdf}, we have
\[
F^{-1}(u)=\left(\frac{u}{C(q)}\right)^\alpha(1+o(1)),
\qquad u\downarrow0.
\]
Because \(\Gamma_{n+1}/n\to1\) almost surely and
\(\beta_n n^{-\alpha}\to\gamma\), it follows that for every fixed \(i\),
\[
\beta_nT_{(i)}
\to
Y_i:=\gamma\left(\frac{\Gamma_i}{C(q)}\right)^\alpha
\qquad\text{a.s.}
\]

It remains to pass this convergence through the softmax denominator. By \Cref{lem:inverse_cdf}, there is a \(c>0\) such that \(F^{-1}(u)\ge cu^\alpha, \forall u \in [0,1]\). On the event \(\Gamma_{n+1}/n\to1\), for all large \(n\),
\[
\beta_nF^{-1}\left(\frac{\Gamma_i}{\Gamma_{n+1}}\right)
\ge c'\Gamma_i^\alpha \, ,
\]
for some $c'$ independent of $n$. We obtain the domination
\[
\exp\left\{
-\beta_nF^{-1}\left(\frac{\Gamma_i}{\Gamma_{n+1}}\right)
\right\}
\le
e^{-c'\Gamma_i^\alpha},
\qquad 1\le i\le n,
\]
for all sufficiently large \(n\). 

We now check that this dominating sequence is summable almost surely. By the
strong law of large numbers,
\[
\frac{\Gamma_i}{i}
=
\frac{E_1+\cdots+E_i}{i}
\longrightarrow 1
\qquad\text{a.s.}
\]
Therefore, almost surely, there exists \(i_0\) such that $\Gamma_i\ge \frac{i}{2}$, for all $i\ge i_0$.
Consequently, for all \(i\ge i_0\), $e^{-c'\Gamma_i^\alpha}\le e^{-c'2^{-\alpha}i^\alpha}$.
Hence
\[
\sum_{i\ge1}e^{-c'\Gamma_i^\alpha}<\infty
\qquad\text{a.s.}
\]

We may therefore apply dominated convergence for series, and get
\[
\sum_{i=1}^n
\exp\left\{
-\beta_nF^{-1}\left(\frac{\Gamma_i}{\Gamma_{n+1}}\right)
\right\}
\longrightarrow
\sum_{i=1}^\infty e^{-Y_i}
\qquad\text{a.s.}
\]
The limiting denominator is finite by the previous summability argument and is
strictly positive because \(e^{-Y_1}>0\). Therefore, for every fixed \(k\),
\[
(A_{(1)},\ldots,A_{(k)})
\Rightarrow
\left(
\frac{e^{-Y_1}}{\sum_{j\ge1}e^{-Y_j}},
\ldots,
\frac{e^{-Y_k}}{\sum_{j\ge1}e^{-Y_j}}
\right).
\]

Finally, we identify the limit with the Poisson point process: the variables \(Y_i=\gamma(\Gamma_i/C(q))^\alpha\) are the increasing
atoms of the Poisson point process on \(\mathbb R_+\) with intensity measure
\(
\Lambda([0,y])=C(q)\left(y/\gamma\right)^{1/\alpha}.
\)
This proves \Cref{thm:intro critical}.

%%%%%%%%%%%%%%%%%%%%%%%%%%%%%%

\section{Subcritical Regime: Proof of \Cref{thm:intro sub_critical 2 new}}\label{app:proof sub_critical}

This appendix proves the attention weight results in the subcritical regime. 

\subsection{Subcritical partition function}

Recall from Appendix~\ref{app:appendix_preliminaries}. We use
\[
    Z_n=\sum_{i=1}^n e^{-\beta_nT_i},
    \qquad
    A_i=\frac{e^{-\beta_nT_i}}{Z_n}.
\]

The following lemma characterizes the limit of $Z_n$ in the subcritical regime.
\begin{lemma}\label{lem:subcritical_scalar_laplace}
As \(\beta\to\infty\),
\[
    \mathbb E[e^{-\beta T_1}]
    =
    C(q)\Gamma\!\left(\frac1\alpha+1\right)
    \beta^{-1/\alpha}(1+o(1)).
\]
\end{lemma}

\begin{proof}
Let \(F(t)=\mathbb P(T_1\le t)\). Integration by parts gives
\[
    \mathbb E[e^{-\beta T_1}]
    =
    e^{-2\beta}
    +
    \beta\int_0^2 e^{-\beta t}F(t)\,dt .
\]
\Cref{lem:small_cap_asymptotics} gives
\[
    F(t)=C(q)t^{1/\alpha}(1+o(1)),
    \qquad t\to0^+.
\]
After the change of variables \(u=\beta t\), dominated convergence yields
\[
    \beta\int_0^{2}e^{-\beta t}F(t)\,dt
    =
    C(q)\beta^{-1/\alpha}
    \int_0^\infty e^{-u}u^{1/\alpha}\,du\,(1+o(1)).
\]
The integral is \(\Gamma(1/\alpha+1)\), which proves the claim.
\end{proof}

\begin{corollary}[Subcritical partition function]\label{cor:subcritical_partition}
If \(\beta_n\to\infty\) and \(\beta_n n^{-\alpha}\to0\), then
\[
    \frac{Z_n}{\mathbb E[Z_n]}\xrightarrow{\mathbb P}1,
    \qquad
    \frac{Z_n}{C(q)n\beta_n^{-1/\alpha}\Gamma\!\left(\frac1\alpha+1\right)}
    \xrightarrow{\mathbb P}1.
\]
\end{corollary}

\begin{proof}
By Lemma \ref{lem:subcritical_scalar_laplace},
\[
    \mathbb E[Z_n]
    =
    C(q)n\beta_n^{-1/\alpha}\Gamma\!\left(\frac1\alpha+1\right)(1+o(1)),
    \qquad
    \mathrm{Var}(Z_n)
    \le n\mathbb E[e^{-2\beta_nT_1}]
    =
    O(n\beta_n^{-1/\alpha}).
\]
Therefore
\[
    \frac{\mathrm{Var}(Z_n)}{\mathbb E[Z_n]^2}
    =
    O\!\left((n\beta_n^{-1/\alpha})^{-1}\right)\to0,
\]
because \(\beta_n n^{-\alpha}\to0\). Chebyshev's inequality gives \(Z_n/\mathbb E[Z_n]\to1\) in probability, and the displayed deterministic asymptotic for \(\mathbb E[Z_n]\) gives the second claim.
\end{proof}

\subsection{Proof of \Cref{thm:intro sub_critical 2 new}}

The ordered-weight proof uses the quantile asymptotics from \Cref{lem:inverse_cdf} to handle the denominator $Z_n$ in $A_{(i)}$. To handle the numerator $e^{-\beta_n T_{(i)}}$ in $A_{(i)}$, we prove in the following that the best D2Q is so small that \(\beta_nT_{(1)}\) vanishes, and the \(k_n\)-th D2Q is governed by the quantile level \(k_n/(n\beta_n^{-1/\alpha})\).
\begin{lemma}\label{lem:subcritical_first_order}
If $\beta_n\to\infty$ and \(\beta_n n^{-\alpha}\to0\), then
\[
    \beta_nT_{(1)}\xrightarrow{\mathbb P}0.
\]
\end{lemma}

\begin{proof}
For every \(\varepsilon>0\),
\[
    \mathbb P(\beta_nT_{(1)}>\varepsilon)
    =
    \left(1-F\!\left(\frac{\varepsilon}{\beta_n}\right)\right)^n
    \le
    \exp\!\left(-nF\!\left(\frac{\varepsilon}{\beta_n}\right)\right).
\]
By \Cref{lem:small_cap_asymptotics},
\[
    nF\!\left(\frac{\varepsilon}{\beta_n}\right)
    =
    C(q)\varepsilon^{1/\alpha}n\beta_n^{-1/\alpha}(1+o(1))\to\infty,
\]
so the probability tends to zero.
\end{proof}

\begin{lemma}\label{lem:subcritical_kth_order}
Assume \(\beta_n\to\infty\), \(\beta_n n^{-\alpha}\to0\), and
\[
    \frac{k_n}{C(q)n\beta_n^{-1/\alpha}}\to x\in[0,\infty).
\]
Then
\[
    \beta_nT_{(k_n)}\xrightarrow{\mathbb P}x^\alpha.
\]
\end{lemma}

\begin{proof}
First suppose \(x>0\). Let \(U_{(k_n)}\) be the \(k_n\)-th order statistic of an i.i.d. uniform sample on \((0,1)\). Then
\[
    T_{(k_n)}\stackrel{d}{=}F^{-1}(U_{(k_n)}).
\]
Since \(k_n\to\infty\) and \(k_n/n\to0\), the beta distribution formulas for uniform order statistics give
\[
    \mathbb E[U_{(k_n)}]=\frac{k_n}{n+1},
    \qquad
    \mathrm{Var}(U_{(k_n)})
    =
    \frac{k_n(n-k_n+1)}{(n+1)^2(n+2)}.
\]
Therefore, 
    \begin{align*}
        \lim_{n \to \infty} \frac{\mathrm{Var}(U_{(k_n)})}{\mathbb E[U_{(k_n)}]^2} = 0.
    \end{align*}
Hence \((n/k_n)U_{(k_n)}\to1\) in probability by Chebyshev's inequality. Applying Lemma \ref{lem:inverse_cdf},
\[
\begin{split}
    \beta_nT_{(k_n)}
    &\stackrel{d}{=}
    \beta_nF^{-1}\!\left(\frac{k_n}{n}\frac{nU_{(k_n)}}{k_n}\right)  \\
    &=
    \beta_n\left(\frac{k_n}{C(q)n}\frac{nU_{(k_n)}}{k_n}\right)^\alpha
    (1+o_{\mathbb P}(1))
    \xrightarrow{\mathbb P}x^\alpha .
\end{split}
\]
A priori the convergence is in distribution, but since the limit is deterministic then $\beta_n T_{(k_n)}$ converges also in probability.
It remains to treat \(x=0\). Fix \(\varepsilon>0\) and set
\[
    M_n(\varepsilon)
    \coloneqq
    \#\left\{i:T_i\le\frac{\varepsilon}{\beta_n}\right\}.
\]
Then \(M_n(\varepsilon)\) is binomial with mean
\[
    \mu_n=nF\!\left(\frac{\varepsilon}{\beta_n}\right)
    =
    C(q)\varepsilon^{1/\alpha}n\beta_n^{-1/\alpha}(1+o(1))\to\infty.
\]
The assumption \(k_n/(C(q)n\beta_n^{-1/\alpha})\to0\) gives \(k_n/\mu_n\to0\). Hence, for all large \(n\), \(k_n\le\mu_n/2\), and Chebyshev's inequality yields
\[
    \mathbb P(\beta_nT_{(k_n)}>\varepsilon)
    =
    \mathbb P(M_n(\varepsilon)<k_n)
    \le
    \mathbb P(|M_n(\varepsilon)-\mu_n|>\mu_n/2)
    \le
    \frac{4\mathrm{Var}(M_n(\varepsilon))}{\mu_n ^2} \leq \frac{4}{\mu_n}\to0,
\]
where we used the fact that $\mathrm{Var}(M_n(\varepsilon)) \leq \mu_n$ because $M_n(\varepsilon)$ is binomial. 
Since \(\varepsilon>0\) was arbitrary, \(\beta_nT_{(k_n)}\to0\) in probability.
\end{proof}

\begin{proof}[Proof of \Cref{thm:intro sub_critical 2 new}]
First, \(A_{(i)}\le A_{(1)}\). To prove \(A_{(i)}\to0\) for fixed \(i\), it suffices to consider subsequences with either \(\beta_n\le M<\infty\) or \(\beta_n\to\infty\). If \(\beta_n\le M\), then
\[
    A_{(1)}
    \le
    \frac{1}{\sum_{j=1}^n e^{-\beta_nT_j}}
    \le
    \frac{e^{2M}}{n}\to0,
\]
since \(T_j\in[0,2]\). If \(\beta_n\to\infty\), then \Cref{cor:subcritical_partition} implies \(Z_n\to\infty\) in probability, and therefore \(A_{(1)}\le1/Z_n\to0\) in probability.

Now assume \(\beta_n\to\infty\) and
\[
    \frac{\beta_n^{1/\alpha} k_n}{C(q)n } \to x,
    \qquad x\ge0.
\]
By Lemma \ref{lem:subcritical_first_order}, and Lemma \ref{lem:subcritical_kth_order},
\[
    \beta_nT_{(1)}\xrightarrow{\mathbb P}0,
    \qquad
    \beta_nT_{(k_n)}
    \xrightarrow{\mathbb P}
    x^\alpha .
\]
Thus
\[
    \frac{A_{(k_n)}}{A_{(1)}}
    =
    \exp\!\left(-\beta_nT_{(k_n)}+\beta_nT_{(1)}\right)
    \xrightarrow{\mathbb P}
    \exp\!\left(-x^\alpha\right).
\]
For the absolute normalization,
\[
    C(q)n\beta_n^{-1/\alpha}\Gamma\!\left(\frac1\alpha+1\right)A_{(k_n)}
    =
    \frac{C(q)n\beta_n^{-1/\alpha}\Gamma\!\left(\frac1\alpha+1\right)}{Z_n}
    e^{-\beta_nT_{(k_n)}}.
\]
The first factor converges to \(1\) in probability by Corollary \ref{cor:subcritical_partition}; the second factor converges to \(\exp\{-x^\alpha\}\) by Lemma \ref{lem:subcritical_kth_order}. Slutsky's theorem gives
\[
    C(q)n\beta_n^{-1/\alpha}\Gamma\!\left(\frac1\alpha+1\right)A_{(k_n)}
    \xrightarrow{\mathbb P}
    \exp\!\left(-x^\alpha\right),
\]
which is equivalent to the stated formula.

Finally, for \eqref{eq:subcritical_cumulative_mass_limit}, by assumption, we have that \(m_n\to\infty\) and
\(
    \lim_{n \to \infty} \tfrac{k_n}{m_n(q)}  = x \geq 0.
\)
We can rewrite the sum in \eqref{eq:subcritical_cumulative_mass_limit} as:
\[
    \sum_{i=1}^{k_n}A_{(i)}
    =
    \frac{m_n}{Z_n}
    \int_0^{k_n/m_n}
    e^{-\beta_nT_{(\lceil sm_n\rceil)}}\,\mathrm{d}s .
\]

We now identify the limit of the integral. For each fixed \(s\in(0,x]\),
\(
    \frac{\lceil sm_n\rceil}{m_n}\to s .
\)
Therefore, by \Cref{lem:subcritical_kth_order},
\(
    \beta_nT_{(\lceil sm_n\rceil)}
    \xrightarrow{\mathbb P}
    s^\alpha,
\)
and hence
\(
    e^{-\beta_nT_{(\lceil sm_n\rceil)}}
    \xrightarrow{\mathbb P}
    e^{-s^\alpha}.
\)
Since the random variables are bounded by \(1\), this convergence also holds in
\(L^1\). Thus, by Fubini and dominated convergence,
\[
    \mathbb E
    \int_0^x
    \left|
        e^{-\beta_nT_{(\lceil s m_n\rceil)}}-e^{-s^\alpha}
    \right|\,\mathrm{d}s
    =
    \int_0^x
    \mathbb E
    \left|
        e^{-\beta_nT_{(\lceil s m_n\rceil)}}-e^{-s^\alpha}
    \right|\,\mathrm{d}s 
    \longrightarrow 0 .
\]
Consequently,
\[
    \int_0^x e^{-\beta_nT_{(\lceil sm_n\rceil)}}\,\mathrm{d}s
    \xrightarrow{L^1}
    \int_0^x e^{-s^\alpha}\,\mathrm{d}s,
\]
and hence also in probability. Moreover,
\[
\left|
    \int_0^{k_n/m_n} e^{-\beta_nT_{(\lceil sm_n\rceil)}}\,\mathrm{d}s
    -
    \int_0^x e^{-\beta_nT_{(\lceil sm_n\rceil)}}\,\mathrm{d}s
\right|
\le
\left|
    \frac{k_n}{m_n}-x
\right|
\to0.
\]
Therefore
\[
    \int_0^{k_n/m_n}
    e^{-\beta_nT_{(\lceil sm_n\rceil)}}\,\mathrm{d}s
    \xrightarrow{\mathbb P}
    \int_0^x e^{-s^\alpha}\,\mathrm{d}s.
\]

Furthermore, by \Cref{cor:subcritical_partition},
\(
    \frac{m_n}{Z_n}
    \xrightarrow{\mathbb P}
    \frac{1}{\Gamma\!\left(\frac1\alpha+1\right)}
\), and using the Continuous Mapping Theorem we get:
\[
    \sum_{i=1}^{k_n}A_{(i)}
    \xrightarrow{\mathbb P}
    \frac{1}{\Gamma\!\left(\frac1\alpha+1\right)}
    \int_0 ^x e^{-s^\alpha}\,\mathrm{d}s= \frac{1}{\Gamma\!\left(\frac1\alpha+1\right)}\frac1\alpha
    \int_0^{x^\alpha}e^{-u}u^{1/\alpha-1}\,\mathrm{d}u=\frac{
        \gamma_{\rm inc}\!\left(\frac1\alpha,x^\alpha\right)
    }{
        \Gamma\!\left(\frac1\alpha\right)
    }.  
\]
\end{proof}

\section{Attention Outputs in the Supercritical, Critical, and Subcritical Regimes}\label{app:attention_outputs}

This appendix proves the attention-output limits from \Cref{sec:attention_outputs}.  
The proofs of the supercritical and critical regimes first identify the relevant local point process near the query, then apply the softmax functional to that process.  
The subcritical regime is different and is treated later by local moment estimates and a triangular-array central limit theorem.

\subsection{Convergence of a marked point process}

We recall from \Cref{sec:attention_outputs},
\[
    a_n= (C(q)n)^\alpha,
    \qquad
    \proj_q= \Id-qq^\top,
\]
and, whenever \(\proj_qx_i\neq0\), define the tangent direction
\[
    \Theta_i\coloneqq \frac{\proj_qx_i}{\|\proj_qx_i\|}
    \in \mathbb S_q^{d-2},
    \qquad
    \mathbb S_q^{d-2}\coloneqq\{\theta\in T_q\mathbb S^{d-1}:\|\theta\|=1\}.
\]
On the null event \(\proj_qx_i=0\), choose \(\Theta_i\) arbitrarily.  Let \(\nu_q\) denote the uniform probability measure on \(\mathbb S_q^{d-2}\).

We use the following marked version of the point-process limit.  
It is the only local extreme-value input needed in \Cref{thm:output_supercritical} and \Cref{thm:output_critical}.

\begin{lemma}[Local marked point process]\label{lem:output_marked_ppp}\label{lem:critical_marked_ppp}
Let \(b_n>0\) satisfy \(b_n n^{-\alpha}\to\lambda\in(0,\infty)\), and define
\[
    \mathcal M_n^{(b)}\coloneqq
    \sum_{i=1}^n \delta_{(b_nT_i,\Theta_i)}
    \qquad\text{on }[0,\infty)\times \mathbb S_q^{d-2}.
\]
Then
\[
    \mathcal M_n^{(b)}\Rightarrow \mathcal M_\lambda,
\]
where \(\mathcal M_\lambda\) is a Poisson point process with intensity measure
\[
    \Lambda_\lambda([0,y]\times B)
    =
    C(q)\left(\frac{y}{\lambda}\right)^{1/\alpha}\nu_q(B)
\]
for every \(y\ge0\) and every \(\nu_q\)-continuity set \(B\subset\mathbb S_q^{d-2}\).
\end{lemma}

\begin{proof}
Let
\[
    \mu_n(A)
    \coloneqq
    n\,\mathbb P\bigl((b_nT_1,\Theta_1)\in A\bigr)
\]
be the one-particle mean measure.  We first show \(\mu_n\to\Lambda_\lambda\) vaguely.  Let \(g\in C_c([0,\infty)\times\mathbb S_q^{d-2})\), with support contained in \([0,R]\times\mathbb S_q^{d-2}\).  In geodesic polar coordinates around \(q\),
\[
    x=(\cos r)q+(\sin r)\theta,
    \qquad
    T(x)=1-\cos r,
    \qquad
    \mathrm{d}\sigma(x)=(\sin r)^{d-2}\mathrm{d}r\,\mathrm{d}\sigma_q(\theta),
\]
where \(\mathrm{d}\sigma_q\) is surface measure on \(\mathbb S_q^{d-2}\).  With the change of variables \(y=b_n(1-\cos r)\),
\[
    (\sin r)^{d-2}\mathrm{d}r
    =
    b_n^{-1/\alpha}
    \left(2y-\frac{y^2}{b_n}\right)^{1/\alpha-1}\mathrm{d}y.
\]
Therefore
\[
\begin{aligned}
    \int g\,\mathrm{d}\mu_n
    &=
    n b_n^{-1/\alpha}
    \int_{\mathbb S_q^{d-2}}\int_0^R
    g(y,\theta)
    \rho\left(
        \left(1-\frac y{b_n}\right)q
        +b_n^{-1/2}\left(2y-\frac{y^2}{b_n}\right)^{1/2}\theta
    \right) \\
    &\hspace{4.6cm}\times
    \left(2y-\frac{y^2}{b_n}\right)^{1/\alpha-1}
    \mathrm{d}y\,\mathrm{d}\sigma_q(\theta).
\end{aligned}
\]
Since \(n b_n^{-1/\alpha}\to\lambda^{-1/\alpha}\), \(\rho\) is continuous at \(q\), and \(y^{1/\alpha-1}\) is integrable at zero, dominated convergence gives
\[
    \int g\,\mathrm{d}\mu_n
    \longrightarrow
    \frac{C(q)}{\alpha\lambda^{1/\alpha}}
    \int_{\mathbb S_q^{d-2}}\int_0^R
    g(y,\theta)y^{1/\alpha-1}\mathrm{d}y\,\nu_q(\mathrm{d}\theta).
\]
Here we used
\[
    C(q)=\frac{2^{1/\alpha}\sigma_{d-2}}{d-1}\rho(q),
    \qquad
    \alpha=\frac{2}{d-1},
    \qquad
    \mathrm{d}\sigma_q=\sigma_{d-2}\,\mathrm{d}\nu_q.
\]
Thus \(\mu_n\to\Lambda_\lambda\) vaguely.

Now take \(\varphi\in C_c([0,\infty)\times\mathbb S_q^{d-2})\), \(\varphi\ge0\).  Since the marked pairs \((b_nT_i,\Theta_i)\) are i.i.d.,
\[
\begin{aligned}
    \mathbb E\exp\left\{-\int\varphi\,\mathrm{d}\mathcal M_n^{(b)}\right\}
    &=
    \left(\mathbb E e^{-\varphi(b_nT_1,\Theta_1)}\right)^n \\
    &=
    \left(
        1-\frac1n\int(1-e^{-\varphi})\,\mathrm{d}\mu_n
    \right)^n.
\end{aligned}
\]
By vague convergence of \(\mu_n\), the integral converges to
\(
\int(1-e^{-\varphi})\,\mathrm{d}\Lambda_\lambda
\).
Hence the Laplace functionals converge to those of a Poisson point process with intensity \(\Lambda_\lambda\).
\end{proof}

\subsection{Supercritical output: proof of \Cref{thm:output_supercritical}}

In the supercritical regime, the softmax output is asymptotically the nearest token itself.  The only extra point beyond \Cref{thm:intro super_critical} is that this replacement is valid at the finer scale \(a_n^{-1/2}\)
, where we recall the definition $a_n = (C(q)n)^\alpha$.

Let
\[
    I_n\coloneqq\arg\min_{1\le i\le n}T_i,
    \qquad
    T_n^\star\coloneqq T_{I_n}=T_{(1)},
    \qquad
    x_n^\star\coloneqq x_{I_n},
    \qquad
    \Theta_n^\star\coloneqq \Theta_{I_n}.
\]
The minimizer is unique almost surely because the law of \(T_i\) is continuous.

\begin{lemma}[Nearest marked token]\label{lem:app_nearest_defect}
Under \Cref{a:iid},
\[
    (a_nT_n^\star,\Theta_n^\star)
    \Rightarrow
    (R_1,\Theta_1),
\]
where
\[
    \mathbb P(R_1>r)=e^{-r^{1/\alpha}},\qquad r\ge0,
\]
and \(\Theta_1\sim\nu_q\) is independent of \(R_1\).
\end{lemma}

\begin{proof}
Apply \Cref{lem:output_marked_ppp} with \(b_n=a_n\).  Since \(a_n n^{-\alpha}=C(q)^\alpha\), the limiting intensity is
\[
    \Lambda_\star(\mathrm{d}y,\mathrm{d}\theta)=\frac1\alpha y^{1/\alpha-1}\mathrm{d}y\,\nu_q(\mathrm{d}\theta).
\]
The first atom map is continuous at the limiting process almost surely, because the process is simple and its \(y\)-marginal is diffuse.  Thus the first marked atom converges.  Its radial tail is
\[
    \mathbb P(R_1>r)=\exp\{-\Lambda_\star([0,r]\times\mathbb S_q^{d-2})\}=e^{-r^{1/\alpha}},
\]
and the product form of \(\Lambda_\star\) gives independence of \(R_1\) and \(\Theta_1\).
\end{proof}

\begin{lemma}[Nearest-token reduction]\label{lem:supercritical_nearest_reduction}
If \(\beta_n n^{-\alpha}\to\infty\), then
\[
    \sqrt{a_n}\left\|\sum_{i=1}^n A_i x_i-x_n^\star\right\|
    \xrightarrow{\mathbb P}0.
\]
\end{lemma}

\begin{proof}
Set \(\lambda_n\coloneqq\beta_n/a_n\), so \(\lambda_n\to\infty\).  Since the denominator of the softmax is at least \(e^{-\beta_nT_n^\star}\),
\[
    \sqrt{a_n}\left\|\sum_{i=1}^n A_i x_i-x_n^\star\right\|
    \le
    D_n
    \coloneqq
    \sum_{i\ne I_n} e^{-\beta_n(T_i-T_n^\star)}
    \sqrt{a_n}\|x_i-x_n^\star\|.
\]
Fix \(M>0\) and write \(E_{n,M}=\{a_nT_n^\star\le M\}\).  By \Cref{lem:app_nearest_defect},
\[
    \lim_{n\to\infty}\mathbb P(E_{n,M}^c)=e^{-M^{1/\alpha}}.
\]
We show that \(\mathbb E[D_n\mathbf 1_{E_{n,M}}]\to0\).  Conditionally on \(T_n^\star=t\), the remaining D2Qs have law \(\mathrm{d}F(s)/(1-F(t))\) on \([t,2]\).  Also
\[
    \|x_i-x_n^\star\|
    \le \|x_i-q\|+\|x_n^\star-q\|
    =\sqrt{2T_i}+\sqrt{2t}.
\]
For \(t\le M/a_n\), and for all large \(n\), \(1-F(t)\ge1/2\).  Using the local density bound \(F'(s)\le Cs^{1/\alpha-1}\) in \Cref{lem:small_cap_asymptotics}, $\mathbb E[D_n\mathbf 1_{E_{n,M}}]$ is bounded by
\[
\begin{aligned}
    C n\sqrt{a_n}
    \int_t ^{2} e^{-\beta_n(s-t)}(\sqrt s+\sqrt t)s^{1/\alpha-1}\,\mathrm{d}s.
\end{aligned}
\]
Put \(t=y/a_n\), \(s=t+v/a_n\).  Since \(0\le y\le M\) and \(n a_n^{-1/\alpha}=C(q)^{-1}\), the last display is at most a constant times
\[
    \sup_{0\le y\le M}
    \int_0^\infty
    e^{-\lambda_n v}
    \bigl(\sqrt{y+v}+\sqrt y\bigr)(y+v)^{1/\alpha-1}\,\mathrm{d}v.
\]
This supremum tends to zero as \(\lambda_n\to\infty\): for \(d=2\) the integrand is bounded by a constant, and for \(d\ge3\) it is bounded on \(0\le y\le M\) by \(C_M(1+v^{1/\alpha-1/2})\).  Hence \(\mathbb E[D_n\mathbf 1_{E_{n,M}}]\to0\).

By Markov's inequality,
\[
    \limsup_{n\to\infty}\mathbb P(D_n>\varepsilon)
    \le e^{-M^{1/\alpha}},
\]
for every \(\varepsilon>0\).  Letting \(M\to\infty\) proves \(D_n\to0\) in probability.
\end{proof}

\begin{proof}[Proof of \Cref{thm:output_supercritical}]
Write
\[
    x_n^\star=(\cos r_n^\star)q+(\sin r_n^\star)\Theta_n^\star,
    \qquad
    T_n^\star=1-\cos r_n^\star.
\]
Then
\[
    x_n^\star-q
    =
    -T_n^\star q+
    \sqrt{2T_n^\star-(T_n^\star)^2}\,\Theta_n^\star.
\]
By \Cref{lem:app_nearest_defect}, \(a_nT_n^\star\) is tight, and therefore
\[
    \sqrt{a_n}(x_n^\star-q)
    =
    \sqrt{2a_nT_n^\star}\,\Theta_n^\star+o_{\mathbb P}(1)
    \Rightarrow
    \sqrt{2R_1}\,\Theta_1.
\]
By \Cref{lem:supercritical_nearest_reduction},
\[
    \sqrt{a_n}\sum_{i=1}^n A_i(x_i-q)
    =
    \sqrt{a_n}(x_n^\star-q)+o_{\mathbb P}(1).
\]
Applying the fixed linear map \(V\) gives
\[
    \sqrt{a_n}\,\mathcal Y_n(q)
    \Rightarrow
    V(\sqrt{2R_1}\Theta_1)=V\Phi.
\]
Since \(a_n\to\infty\), this also implies \(\mathcal Y_n(q)\to0\) in probability and hence \(\att(n)\to Vq\) in probability.
\end{proof}

\subsection{Critical output: proof of \Cref{thm:output_critical}}

In the critical regime, the output is no longer determined by a single nearest token.  Instead, the finitely many near-query competitors surviving in the Poisson limit all contribute through the same softmax functional as in Appendix~\ref{app:proof critical}.

Assume \(\beta_n n^{-\alpha}\to\gamma\in(0,\infty)\).  Let
\[
    \mathcal M_n\coloneqq\sum_{i=1}^n\delta_{(\beta_nT_i,\Theta_i)}.
\]
By \Cref{lem:output_marked_ppp} with \(b_n=\beta_n\),
\[
    \mathcal M_n\Rightarrow\mathcal M_\gamma,
    \qquad
    \Lambda_\gamma(\mathrm{d}y,\mathrm{d}\theta)
    =
    \frac{C(q)}{\alpha\gamma^{1/\alpha}}y^{1/\alpha-1}\,\mathrm{d}y\,\nu_q(\mathrm{d}\theta).
\]
Define
\[
    H_n\coloneqq \int_{0} ^{\infty} \int_{\mathbb S_q^{d-2}} e^{-y}\sqrt{2y}\,\theta\,\mathcal M_n(\mathrm{d}y,\mathrm{d}\theta),
    \qquad
    Z_n\coloneqq\int_{0} ^{\infty} \int_{\mathbb S_q^{d-2}} e^{-y}\,\mathcal M_n(\mathrm{d}y,\mathrm{d}\theta),
\]
and define \(H\) and \(Z\) by the same formulas with \(\mathcal M_\gamma\) in place of \(\mathcal M_n\).

\begin{lemma}[Critical softmax functional]\label{lem:critical_output_softmax_functional}
We have
\[
    (H_n,Z_n)\Rightarrow(H,Z),
    \qquad
    0<Z<\infty\quad\text{almost surely}.
\]
Consequently,
\[
    \frac{H_n}{Z_n}\Rightarrow \frac{H}{Z}.
\]
\end{lemma}
\begin{proof}
For fixed \(R<\infty\), the truncated maps
\[
    \mu\mapsto\int_{0} ^R \int_{\mathbb S_q^{d-2}} e^{-y}\sqrt{2y}\,\theta\,\mu(\mathrm{d}y,\mathrm{d}\theta),
    \qquad
    \mu\mapsto\int_{0} ^R \int_{\mathbb S_q^{d-2}} e^{-y}\,\mu(\mathrm{d}y,\mathrm{d}\theta)
\]
are continuous at every point measure with no atom on \(\{R\}\times\mathbb S_q^{d-2}\).  Since the limiting intensity has diffuse \(y\)-marginal, the continuous mapping theorem gives convergence of the truncated pair.

It remains to remove the truncation.  The required tail bound is
\begin{equation}\label{eq:critical_marked_tail_bound}
    \lim_{R\to\infty}\limsup_{n\to\infty}
    \mathbb E\int_{R} ^{\infty} \int_{\mathbb S_q^{d-2}}e^{-y}(1+\sqrt y)\,\mathcal M_n(\mathrm{d}y,\mathrm{d}\theta)=0,
\end{equation}
and the same bound holds with \(\mathcal M_\gamma\) in place of \(\mathcal M_n\).  Indeed,
\[
\begin{aligned}
    \mathbb E\int_{R} ^{\infty} \int_{\mathbb S_q^{d-2}}e^{-y}(1+\sqrt y)\,\mathcal M_n(\mathrm{d}y,\mathrm{d}\theta)
    &=
    n\mathbb E\bigl[e^{-\beta_nT_1}(1+\sqrt{\beta_nT_1})
        \mathbf 1_{\{\beta_nT_1>R\}}\bigr] \\
    &\le
    \sum_{\ell=\lfloor R\rfloor}^\infty
    e^{-\ell}(1+\sqrt{\ell+1})
    n\mathbb P(\beta_nT_1\le \ell+1).
\end{aligned}
\]
For \(\ell+1\le t_0\beta_n\), the small-cap bound gives
\[
    n\mathbb P(\beta_nT_1\le\ell+1)
    \le C\gamma^{-1/\alpha}(\ell+1)^{1/\alpha},
\]
while the remaining range \(\ell +1>t_0\beta_n\) is exponentially small.  Hence the limsup is bounded by the tail of a convergent exponential series.  The limiting version follows directly from Campbell's theorem:
\[
    \mathbb E\int_R ^{\infty} \int_{\mathbb S_q^{d-2}}e^{-y}(1+\sqrt y)\,\mathcal M_\gamma(\mathrm{d}y,\mathrm{d}\theta)
    =
    \frac{C(q)}{\alpha\gamma^{1/\alpha}}
    \int_R^\infty e^{-y}(1+\sqrt y)y^{1/\alpha-1}\,\mathrm{d}y\to0.
\]
The same \(\varepsilon/3\) truncation argument used in Appendix~\ref{app:proof critical} now yields \((H_n,Z_n)\Rightarrow(H,Z)\).

Finally, \(Z<\infty\) almost surely by Campbell's theorem.  Also, for every \(R>0\),
\[
    \mathbb P(Z=0)
    \le
    \mathbb P\bigl(\mathcal M_\gamma([0,R]\times\mathbb S_q^{d-2})=0\bigr)
    =
    \exp\left\{-C(q)\left(\frac R\gamma\right)^{1/\alpha}\right\},
\]
which tends to zero as \(R\to\infty\).  Thus \(Z>0\) almost surely, and the ratio convergence follows from the continuous mapping theorem.
\end{proof}

\begin{proof}[Proof of \Cref{thm:output_critical}]
The ideal tangent numerator is \(H_n\), but the true centered numerator is
\[
    \widehat H_n
    \coloneqq
    \sum_{i=1}^n e^{-\beta_nT_i}\sqrt{\beta_n}(x_i-q),
    \qquad
    \sqrt{\beta_n}\sum_{i=1}^n A_i(x_i-q)
    =
    \frac{\widehat H_n}{Z_n}.
\]
We first show that
\begin{equation}\label{eq:critical_true_ideal_numerator}
    \widehat H_n-H_n\xrightarrow{\mathbb P}0.
\end{equation}
For an atom with \(y_i=\beta_nT_i\), write
\[
    x_i=
    \left(1-\frac{y_i}{\beta_n}\right)q
    +
    \left(\frac{2y_i}{\beta_n}-\frac{y_i^2}{\beta_n^2}\right)^{1/2}\Theta_i.
\]
Therefore, uniformly over \(0\le y_i\le R\),
\[
    \left\|
        \sqrt{\beta_n}(x_i-q)-\sqrt{2y_i}\,\Theta_i
    \right\|	=o_{\mathbb P}(1).
\]
The contribution from \(\{y_i\le R\}\) is thus \(o_{\mathbb P}(1)\), because \(\mathcal M_n([0,R]\times\mathbb S_q^{d-2})\) is tight.  On the tail \(\{y_i>R\}\),
\[
\begin{aligned}
    \left\|
        \sum_{y_i>R}e^{-y_i}
        \bigl(\sqrt{\beta_n}(x_i-q)-\sqrt{2y_i}\Theta_i\bigr)
    \right\|
    \le
    2\sqrt2\int_R ^{\infty} \int_{\mathbb S_q^{d-2}}  e^{-y}\sqrt y\,\mathcal M_n(\mathrm{d}y,\mathrm{d}\theta),
\end{aligned}
\]
and the expectation of the right-hand side vanishes after \(n\to\infty\) and then \(R\to\infty\) by \eqref{eq:critical_marked_tail_bound}.  This proves \eqref{eq:critical_true_ideal_numerator}.

By \Cref{lem:critical_output_softmax_functional},
\[
    \sqrt{\beta_n}\sum_{i=1}^n A_i(x_i-q)
    =
    \frac{\widehat H_n}{Z_n}
    \Rightarrow
    \frac{H}{Z}.
\]
Applying \(V\) gives
\[
    \sqrt{\beta_n}\,\mathcal Y_n(q)
    \Rightarrow
    V\frac{\displaystyle\int e^{-y}\sqrt{2y}\,\theta\,\mathcal M_\gamma(\mathrm{d}y,\mathrm{d}\theta)}
          {\displaystyle\int e^{-y}\,\mathcal M_\gamma(\mathrm{d}y,\mathrm{d}\theta)}
    =V\Xi.
\]
Since \(\beta_n\to\infty\), this also implies \(\mathcal Y_n(q)\to0\) in probability and hence \(\att(n)\to Vq\) in probability.
\end{proof}

\subsection{Local moment estimates for the attention output}

The attention output theorem is different from the ordered-weight theorem, since it requires the vector numerator, not just the scalar denominator. Write
\[
    B_n(q)\coloneqq \sum_{i=1}^n (x_i-q)e^{-\beta_nT_i},
    \qquad
    Z_n=\sum_{i=1}^n e^{-\beta_nT_i}.
\]
Then, when \(V=\Id\)
\[
    \mathcal Y_n(q)=\frac{B_n(q)}{Z_n}
\]
The following local estimates identify the deterministic bias and covariance of \(B_n(q)\). 
In this subsection, 
\(
    \proj_q\coloneqq \Id-qq^\top
\)
denotes the orthogonal projection onto \(T_q\S\), and \(\nabla_{\S}\) denotes the spherical gradient. We also recall the two deterministic quantities used in \Cref{thm:intro sub_critical 3}:
\[
    \mathcal Y(q)
    =\nabla_{\S}\log\rho(q)-\frac{d-1}{2}q,
    \qquad
    \Sigma(q)=\frac{c_d}{\rho(q)}\proj_q .
\]
The proof below uses the additional \(C^2\)-regularity of \(\rho\) near \(q\) assumed in \Cref{thm:intro sub_critical 3}.

\begin{lemma}[Local output moments]\label{lem:subcritical_output_moments}
As \(\beta\to\infty\),
\begin{align}
    \mathbb E[(x_1-q)e^{-\beta T_1}]
    &=
    C(q)\Gamma\!\left(\frac1\alpha+1\right)
    \beta^{-1/\alpha-1} \mathcal Y(q)(1+o(1)),
    \label{eq:subcritical_moment1}\\
    \mathbb E[(x_1-q)(x_1-q)^\top e^{-2\beta T_1}]
    &=
    C(q)\Gamma\!\left(\frac1\alpha+1\right)
    (2\beta)^{-1/\alpha-1}\,\proj_q(1+o(1)),
    \label{eq:subcritical_moment2}\\
    \mathbb E[\|x_1-q\|^3e^{-3\beta T_1}]
    &=
    O\!\left(\beta^{-1/\alpha-3/2}\right).
    \label{eq:subcritical_moment3}
\end{align}
\end{lemma}

\begin{proof}
Rotate coordinates so that \(q=e_d\), and put \(m=d-1=2/\alpha\). We parametrize a neighborhood of \(q\) by tangent coordinates \(v\in\mathbb R^m\):
\[
    \Psi(v)=\bigl(v,\sqrt{1-\|v\|^2}\bigr),
    \qquad \|v\|<\delta .
\]
In this chart,
\[
    T(\Psi(v))
    =
    1-\langle q,\Psi(v)\rangle
    =
    1-\sqrt{1-\|v\|^2}
    =
    \frac{\|v\|^2}{2}+O(\|v\|^4),
\]
\[
    \mathrm{d}\sigma(\Psi(v))
    =
    (1-\|v\|^2)^{-1/2}\,dv
    =
    (1+O(\|v\|^2))\,dv,
\]
and
\[
    \Psi(v)-q
    =
    \bigl(v,-\|v\|^2/2\bigr)+O(\|v\|^4),
\]
where \(\mathrm{d}\sigma\) denotes surface measure on \(\S\),
Moreover, since \(\rho\) is \(C^2\) near \(q\),
\[
    \rho(\Psi(v))
    =
    \rho(q)+\langle \nabla_{\S}\rho(q),v\rangle+O(\|v\|^2),
\]
where \(T_q\S\) is identified with \(\mathbb R^m\times\{0\}\). The contribution from the complement of this coordinate patch is exponentially small because \(T\) is bounded below there by a positive constant.

For the first moment \eqref{eq:subcritical_moment1}, the tangent component comes from multiplying the linear term \(v\) in \(\Psi(v)-q\) by the linear term \(\langle\nabla_{\S}\rho(q),v\rangle\) in the density expansion:
\[
    \int_{\mathbb R^m}v\,\langle\nabla_{\S}\rho(q),v\rangle \, e^{-\beta\|v\|^2/2}\,\mathrm{d}v
    =
    (2\pi)^{m/2}\beta^{-m/2-1}\nabla_{\S}\rho(q).
\]
The normal component comes from \(-\|v\|^2q/2\), giving
\[
    -\frac{\rho(q)}{2}q
    \int_{\mathbb R^m}\|v\|^2 e^{-\beta\|v\|^2/2}\,\mathrm{d}v
    =
    -\frac{m}{2}\rho(q)(2\pi)^{m/2}\beta^{-m/2-1}q.
\]
All other terms are lower order: odd Gaussian integrals vanish, and the remaining nonzero terms contain at least three tangent powers or the factor \(\beta\|v\|^4\) from the exponential expansion. Thus
\[
    \mathbb E[(x_1-q)e^{-\beta T_1}]
    =
    \rho(q)(2\pi)^{m/2}
    \left(\nabla_{\S}\log\rho(q)-\frac{m}{2}q\right)
    \beta^{-m/2-1}(1+o(1)).
\]
Since 
\[
    C(q)\Gamma\!\left(\frac1\alpha+1\right)=\rho(q)(2\pi)^{1/\alpha},
\]
this proves \eqref{eq:subcritical_moment1}.

For the second moment \eqref{eq:subcritical_moment2} with \(e^{-2\beta T_1}\), the leading contribution is tangent--tangent:
\[
    \rho(q)\int_{\mathbb R^m}vv^\top e^{-\beta\|v\|^2}\,\mathrm{d}v
    =
    \rho(q)(2\pi)^{m/2}(2\beta)^{-m/2-1}\Id_m.
\]
Embedding the tangent covariance into \(\mathbb R^d\) gives \(\proj_q\). 
The normal--normal part is order \(\beta^{-m/2-2}\), and the tangent--normal mixed part is lower order by parity or by degree. This proves \eqref{eq:subcritical_moment2}.

Finally, for the third moment \eqref{eq:subcritical_moment3}, in the same chart, \(\|x-q\|\asymp\|v\|\) and \(T(\Psi(v))\asymp\|v\|^2\). Hence
\[
    \|x-q\|^3e^{-3\beta T(x)}
    \le
    C\|v\|^3e^{-c\beta\|v\|^2}
\]
locally, and so the local contribution is bounded by
\[
    C\int_{\mathbb R^m}\|v\|^3e^{-c\beta\|v\|^2}\,\mathrm{d}v
    =
    O(\beta^{-(m+3)/2})
    =
    O(\beta^{-1/\alpha-3/2}).
\]
The contribution away from the chart is exponentially small, proving \eqref{eq:subcritical_moment3}.
\end{proof}

\subsection{Subcritical output: proof of \Cref{thm:intro sub_critical 3}}

We now prove the output theorem. The proof has two reductions. The first replaces the random denominator \(Z_n\) by its expectation using \Cref{cor:subcritical_partition}. The second applies a triangular-array central limit theorem to the numerator \(B_n(q)\).

\begin{lemma}[Denominator reduction]\label{lem:subcritical_ratio_reduction}
Assume \(V=\Id\). If \(r_n>0\) and
\[
    r_n\frac{B_n(q)}{\mathbb E[Z_n]}
    \Rightarrow Z,
\]
then
\[
    r_n\mathcal Y_n(q)\Rightarrow Z.
\]
The same implication holds with convergence in probability in place of convergence in distribution.
\end{lemma}

\begin{proof}
By \Cref{cor:subcritical_partition},
\[
    \frac{Z_n}{\mathbb E[Z_n]}\xrightarrow{\mathbb P}1,
    \qquad
    \frac{\mathbb E[Z_n]}{Z_n}\xrightarrow{\mathbb P}1.
\]
Since
\[
    r_n\mathcal Y_n(q)
    =
    r_n\frac{B_n(q)}{\mathbb E[Z_n]}
    \cdot
    \frac{\mathbb E[Z_n]}{Z_n},
\]
Slutsky's theorem finishes the proof.
\end{proof}

\begin{lemma}[Numerator fluctuation lemma]\label{lem:subcritical_numerator_fluctuations}
Set
\[
    X_{n,i}\coloneqq (x_i-q)e^{-\beta_nT_i},
    \qquad
    m_n\coloneqq \mathbb E[B_n(q)]=n\mathbb E[X_{n,1}],
    \qquad
    \Gamma_n\coloneqq n\,\mathrm{Cov}(X_{n,1}).
\]
Suppose \(r_n>0\), \(\mu\in\mathbb R^d\), and \(\Gamma\) is a symmetric nonnegative matrix such that
\begin{align}
    r_n\frac{m_n}{\mathbb E[Z_n]} &\to \mu, \label{eq:subcritical_fluct_mean}\\
    \frac{r_n^2}{\mathbb E[Z_n]^2}\Gamma_n &\to \Gamma, \label{eq:subcritical_fluct_cov}\\
    \frac{r_n^3 n\mathbb E[\|X_{n,1}\|^3]}{\mathbb E[Z_n]^3} &\to0. \label{eq:subcritical_fluct_lyap}
\end{align}
Then
\[
    r_n\frac{B_n(q)}{\mathbb E[Z_n]}
    \Rightarrow \mu+\mathcal N(0,\Gamma).
\]
If \(\Gamma=0\), this means convergence in probability to \(\mu\).
\end{lemma}

\begin{proof}
Define
\[
    Y_{n,i}
    \coloneqq
    \frac{r_n}{\mathbb E[Z_n]}
    \bigl(X_{n,i}-\mathbb E[X_{n,1}]\bigr).
\]
Then the \(Y_{n,i}\) are independent and centered, and
\[
    r_n\frac{B_n(q)}{\mathbb E[Z_n]}
    =
    r_n\frac{m_n}{\mathbb E[Z_n]}
    +
    \sum_{i=1}^nY_{n,i}.
\]
Moreover,
\[
    \sum_{i=1}^n\mathrm{Cov}(Y_{n,i})
    =
    \frac{r_n^2}{\mathbb E[Z_n]^2}\Gamma_n\to\Gamma
\]
by \eqref{eq:subcritical_fluct_cov}, and
\[
    \sum_{i=1}^n\mathbb E[\|Y_{n,i}\|^3]
    \le
    C\frac{r_n^3n\mathbb E[\|X_{n,1}\|^3]}{\mathbb E[Z_n]^3}
    \to0
\]
by \eqref{eq:subcritical_fluct_lyap}. The Lyapunov central limit theorem for triangular arrays (for instance, Proposition 2.27 in \cite{van2000asymptotic}) gives
\[
    \sum_{i=1}^nY_{n,i}\Rightarrow\mathcal N(0,\Gamma).
\]
Combining this with \eqref{eq:subcritical_fluct_mean} proves the claim.
\end{proof}

\begin{proof}[Proof of \Cref{thm:intro sub_critical 3}]
It is enough to prove the theorem when \(V=\Id\), because the general case follows by applying the linear map \(V\) to the identity-value limits. We therefore work with
\[
    \mathcal Y_n(q)=\frac{B_n(q)}{Z_n}.
\]

By \Cref{lem:subcritical_scalar_laplace} and \eqref{eq:subcritical_moment1},
\begin{align}
    \mathbb E[Z_n]
    &=
    C(q)\Gamma\!\left(\frac1\alpha+1\right)
    n\beta_n^{-1/\alpha}(1+o(1)), \label{eq:subcritical_ESn}\\
    m_n
    &=
    C(q)\Gamma\!\left(\frac1\alpha+1\right)
    \mathcal Y(q)n\beta_n^{-1/\alpha-1}(1+o(1)).
    \label{eq:subcritical_Emn}
\end{align}
Moreover, \eqref{eq:subcritical_moment2} implies
\[
    n\mathbb E[X_{n,1}X_{n,1}^\top]
    =
    C(q)\Gamma\!\left(\frac1\alpha+1\right)
    2^{-1/\alpha-1}\proj_q\,n\beta_n^{-1/\alpha-1}(1+o(1)).
\]
Since \eqref{eq:subcritical_moment1} gives \(\mathbb E[X_{n,1}]=O(\beta_n^{-1/\alpha-1})\),
\[
    n\mathbb E[X_{n,1}]\mathbb E[X_{n,1}]^\top
    =
    O(n\beta_n^{-2/\alpha-2})
    =
    o(n\beta_n^{-1/\alpha-1}).
\]
Therefore
\[
    \Gamma_n
    =
    C(q)\Gamma\!\left(\frac1\alpha+1\right)
    2^{-1/\alpha-1}\proj_q\,n\beta_n^{-1/\alpha-1}(1+o(1)).
\]
Recalling that \(c_d=2^{-d}\pi^{-(d-1)/2}\) and \(\Sigma(q)=c_d\rho(q)^{-1}\proj_q\), we get
\begin{align}
    \frac{\beta_n^2}{\mathbb E[Z_n]^2}\Gamma_n
    &=
    \Sigma(q)\frac{\beta_n^{1+1/\alpha}}{n}(1+o(1)),
    \label{eq:subcritical_cov_beta}\\
    \frac{\beta_n^3n\mathbb E[\|X_{n,1}\|^3]}{\mathbb E[Z_n]^3}
    &=
    O\!\left(
    \beta_n^{-1/2}
    \left(\frac{\beta_n^{1+1/\alpha}}{n}\right)^2
    \right).
    \label{eq:subcritical_lyap_beta}
\end{align}

We now apply the numerator fluctuation lemma in the three possible subcritical output regimes.

\smallskip
\noindent\emph{Drift-dominated regime.}
Assume \(\beta_n^{1+\alpha}n^{-\alpha}\to0\), equivalently \(\beta_n^{1+1/\alpha}/n\to0\), and set \(r_n=\beta_n\). From \eqref{eq:subcritical_ESn} and \eqref{eq:subcritical_Emn},
\[
    r_n\frac{m_n}{\mathbb E[Z_n]}\to\mathcal Y(q).
\]
By \eqref{eq:subcritical_cov_beta},
\[
    \frac{r_n^2}{\mathbb E[Z_n]^2}\Gamma_n\to0,
\]
and by \eqref{eq:subcritical_lyap_beta},
\[
    \frac{r_n^3n\mathbb E[\|X_{n,1}\|^3]}{\mathbb E[Z_n]^3}\to0.
\]
Thus \Cref{lem:subcritical_numerator_fluctuations,lem:subcritical_ratio_reduction} yield
\[
    \beta_n\mathcal Y_n(q)\xrightarrow{\mathbb P}\mathcal Y(q).
\]

\smallskip
\noindent\emph{Mixed drift--fluctuation regime.}
Assume \(\beta_n^{1+\alpha}n^{-\alpha}\to\tau\in(0,\infty)\), equivalently \(\beta_n^{1+1/\alpha}/n\to\tau^{1/\alpha}\), and again set \(r_n=\beta_n\). Then
\[
    r_n\frac{m_n}{\mathbb E[Z_n]}\to\mathcal Y(q),
    \qquad
    \frac{r_n^2}{\mathbb E[Z_n]^2}\Gamma_n\to\tau^{1/\alpha}\Sigma(q).
\]
Moreover, \eqref{eq:subcritical_lyap_beta} gives
\[
    \frac{r_n^3n\mathbb E[\|X_{n,1}\|^3]}{\mathbb E[Z_n]^3}
    =
    O(\beta_n^{-1/2})\to0,
\]
because \(n\asymp\beta_n^{1+1/\alpha}\) in this regime. Hence
\[
    \beta_n\mathcal Y_n(q)
    \Rightarrow
    \mathcal Y(q)+\mathcal N(0,\tau^{1/\alpha}\Sigma(q)).
\]

\smallskip
\noindent\emph{Fluctuation-dominated regime.}
Assume \(\beta_n^{1+\alpha}n^{-\alpha}\to\infty\) while still \(\beta_n n^{-\alpha}\to0\), and set
\[
    r_n=\sqrt{n\beta_n^{-1/\alpha+1}}.
\]
Then
\[
    r_n\frac{m_n}{\mathbb E[Z_n]}
    =
    \sqrt{n\beta_n^{-1/\alpha-1}}\,
    \mathcal Y(q)(1+o(1))
    \to0,
\]
because \(\beta_n^{1+1/\alpha}/n\to\infty\). Also,
\[
    \frac{r_n^2}{\mathbb E[Z_n]^2}\Gamma_n\to\Sigma(q).
\]
Finally, by \eqref{eq:subcritical_moment3},
\[
    \frac{r_n^3 n\mathbb E[\|X_{n,1}\|^3]}{\mathbb E[Z_n]^3}
    =
    O\!\left((n\beta_n^{-1/\alpha})^{-1/2}\right)\to0.
\]
Thus
\[
    \sqrt{n\beta_n^{-1/\alpha+1}}\,\mathcal Y_n(q)
    \Rightarrow
    \mathcal N(0,\Sigma(q)).
\]

In all three cases the normalizing factor diverges, so \(\mathcal Y_n(q)\to0\) in probability in the identity-value case. For general \(V\), apply the linear map \(z\mapsto Vz\) to the identity-value limits. This gives the three conclusions of \Cref{thm:intro sub_critical 3}, with covariance \(V\Sigma(q)V^\top\), and also \(\mathcal Y_n(q)\to0\) in probability.
\end{proof}

\subsection{Residual dynamics: proof of \Cref{prop:backward_heat}} 

The residual dynamics use only the drift-dominated output regime and the first-order expansion of the normalization map \(u\mapsto u/\|u\|\).

\begin{proof}[Proof of \Cref{prop:backward_heat}]
Because \(\beta_n^{1+\alpha}n^{-\alpha}\to0\), \Cref{thm:intro sub_critical 3}(1) specialized to \(V=\Id\) gives
\[
    \beta_n\mathcal Y_n(q)
    \xrightarrow{\mathbb P}
    \mathcal Y(q)
    =
    \nabla_{\mathbb S^{d-1}}\log\rho(q)-\frac{d-1}{2}q.
\]
Consequently, \(\mathcal Y_n(q)=O_{\mathbb P}(\beta_n^{-1})\).

For \(q\in\S\) and a small perturbation \(h\),
\[
    \|q+h\|^{-1}
    =
    \bigl(1+2\langle q,h\rangle+\|h\|^2\bigr)^{-1/2}
    =
    1-\langle q,h\rangle+O(\|h\|^2).
\]
Thus
\[
    \frac{q+h}{\|q+h\|}
    =
    q+h-\langle q,h\rangle q+O(\|h\|^2)
    =
    q+\proj_{q}h+O(\|h\|^2).
\]
Setting \(h=\gamma\mathcal Y_n(q)\), we obtain
\[
    q'
    =
    q+\gamma\proj_{q}\mathcal Y_n(q)+R_{n},
    \qquad
    R_{n}=O_{\mathbb P}(\beta_n^{-2}).
\]
Since \(\proj_{q}\mathcal Y(q)=\nabla_{\mathbb S^{d-1}}\log\rho(q)\),
\begin{align}
    \beta_n\proj_{q}\mathcal Y_n(q)
    \xrightarrow{\mathbb P}
    \nabla_{\mathbb S^{d-1}}\log\rho(q).
    \label{eq:projected_limit_backward_heat}
\end{align}
and this concludes the proof.
\end{proof}

An immediate consequence is the following corollary:

\begin{corollary}\label{cor:backward_heat_app}
Assume $V=\Id$ and let $q^\ast$ be a critical point of $\rho$ on $\mathbb S^{d-1}$. For any $q\neq q^\ast$ in a sufficiently small neighborhood of $q^\ast$, the limiting update in \Cref{prop:backward_heat} increases $\langle q', (q^\ast)'\rangle$ if $q^\ast$ is a strict local maximum of $\rho$, and decreases it if $q^\ast$ is a strict local minimum.
\end{corollary}
\begin{proof}
Consider $q_1,q_2\in\S$.  Expanding the inner product of the updated queries, and using \Cref{prop:backward_heat} yields:
\begin{align}
\label{eq:scalar_prods}
    \beta_n\bigl(\langle q_1',q_2'\rangle-\langle q_1,q_2\rangle\bigr)
    \xrightarrow{\mathbb P}
    \gamma\Bigl(
        \langle \nabla_{\mathbb S^{d-1}}\log\rho(q_1),q_2\rangle
        +
        \langle q_1,\nabla_{\mathbb S^{d-1}}\log\rho(q_2)\rangle
    \Bigr).
\end{align}
Apply \Cref{eq:scalar_prods} with \(q_1=q\) and \(q_2=q^\ast\). Since \(\nabla_{\mathbb S^{d-1}}\log\rho(q^\ast)=0\),
\[
    \beta_n\bigl(\langle q',(q^\ast)'\rangle-\langle q,q^\ast\rangle\bigr)
    \xrightarrow{\mathbb P}
    \gamma\langle\nabla_{\mathbb S^{d-1}}\log\rho(q),q^\ast\rangle.
\]
It remains to identify the sign of the limiting inner product. Parameterize the geodesic issuing from \(q^\ast\) by
\[
    q(\theta)=(\cos\theta)q^\ast+(\sin\theta)u,
\]
where \(u\in T_{q^\ast}\S\) is a unit tangent vector and \(\theta>0\) is small. Since \(\nabla_{\mathbb S^{d-1}}\log\rho(q(\theta))\) is tangent to the sphere at \(q(\theta)\), it is orthogonal to \(q(\theta)\), so
\[
    \cos\theta\,
    \langle\nabla_{\mathbb S^{d-1}}\log\rho(q(\theta)),q^\ast\rangle
    +
    \sin\theta\,
    \langle\nabla_{\mathbb S^{d-1}}\log\rho(q(\theta)),u\rangle
    =
    0.
\]
Therefore
\[
    \langle\nabla_{\mathbb S^{d-1}}\log\rho(q(\theta)),q^\ast\rangle
    =
    -\tan\theta\,
    \langle\nabla_{\mathbb S^{d-1}}\log\rho(q(\theta)),u\rangle.
\]
A Taylor expansion along the geodesic around the critical point \(q^\ast\) gives
\[
    \langle\nabla_{\mathbb S^{d-1}}\log\rho(q(\theta)),u\rangle
    =
    \theta\,
    \langle\nabla^2_{\mathbb S^{d-1}}\log\rho(q^\ast)u,u\rangle
    +o(\theta).
\]
Using \(\tan\theta=\theta+o(\theta)\), we obtain
\[
    \langle\nabla_{\mathbb S^{d-1}}\log\rho(q(\theta)),q^\ast\rangle
    =
    -\theta^2
    \langle\nabla^2_{\mathbb S^{d-1}}\log\rho(q^\ast)u,u\rangle
    +o(\theta^2).
\]
If \(q^\ast\) is a strict local maximum of \(\rho\), then the spherical Hessian of \(\log\rho\) at \(q^\ast\) is negative definite, so the right-hand side is strictly positive for \(\theta\) small. Thus the limiting update increases \(\langle q',(q^\ast)'\rangle\). Conversely, if \(q^\ast\) is a strict local minimum of \(\rho\), then the spherical Hessian is positive definite, so the limiting update decreases \(\langle q',(q^\ast)'\rangle\). Finally, since \(\nabla_{\mathbb S^{d-1}}\log\rho(q^\ast)=0\), \Cref{eq:scalar_prods} applied with both queries equal to \(q^\ast\) shows that \(q^\ast\) does not move to first order.
\end{proof}

% ----------------------------------------------------------------------
% Appendix proof for the correlated extension
% ----------------------------------------------------------------------

\section{RoPE with Correlated Tokens}
\label{app:proof_correlated_critical}

In this appendix, we consider the critical scaling problem when the attention scores incorporate Rotary Positional Embeddings (RoPE) \cite{su2024roformer} and the context tokens are correlated. 

RoPE encodes the positions by rotating query and key coordinates in frequency-dependent two-dimensional planes before taking their inner product. Since these rotations are orthogonal, a RoPE score can be rewritten as an ordinary dot product between the key and a query direction that moves deterministically with relative position. Thus, for the one-query asymptotic studied in this paper, RoPE changes the fixed target direction $q$ into a deterministic orbit $(u_i)$ on the sphere. The random part remains the token sequence $(x_i)$, while the RoPE rotations are deterministic.

We use the standard block-diagonal RoPE form. For an even rotated dimension $d=2L$, let
\[
    R_p=\operatorname{diag}\bigl(r(p\vartheta_1),\ldots,r(p\vartheta_L)\bigr),
    \qquad
    r(\phi)=
    \begin{pmatrix}
        \cos\phi&-\sin\phi\\
        \sin\phi&\cos\phi
    \end{pmatrix}.
\]
The same notation also covers the common case where only a rotated subspace is used, by leaving the remaining coordinates unrotated. If a query at position $p$ attends to a key at position $i$, then
\[
    \langle R_pq,R_i x_i\rangle
    =\langle R_i^\top R_pq,x_i\rangle
    =\langle R_{i-p}^\top q,x_i\rangle.
\]
After relabeling relative positions, we write
\[
    u_i:=R_{i}^\top q,
    \qquad
    \langle R_0q,R_i x_i\rangle=\langle u_i,x_i\rangle.
\]
Changing the reference position or the sign convention only shifts the initial phase or replaces the frequency vector by its negative, and does not affect the limiting phase law below.

It is useful to encode this deterministic orbit as a torus rotation. Let
\[
    \mathbb T:=\mathbb R/\mathbb Z,
    \qquad
    \theta:=\left(\frac{\vartheta_1}{2\pi},\ldots,\frac{\vartheta_L}{2\pi}\right)
    \in \mathbb T^L.
\]
Define the continuous map
\[
    \Phi_q:\mathbb T^L\longrightarrow \S,
    \qquad
    \Phi_q(x):=
    \operatorname{diag}\bigl(r(-2\pi x_1),\ldots,r(-2\pi x_L)\bigr)q.
\]
Then the RoPE target sequence is simply
\[
    u_i=\Phi_q(i\theta\bmod 1),
    \qquad i\in\mathbb Z.
\]
This is the only structural property of RoPE used in the proof: the target direction follows a deterministic translation orbit on a torus, mapped continuously into the sphere.

Let
\[
    \mathcal{H}_q:=\overline{\{i\theta\bmod 1:i\in\mathbb Z\}}\subset\mathbb T^L.
\]
This is a closed subgroup of the torus. Let $m_{\mathcal{H}_q}$ denote Haar probability measure on $\mathcal{H}_q$ \cite{folland2016course}, and define the deterministic phase law
\[
    \pi_q:=(\Phi_q)_\# m_{\mathcal{H}_q}.
\]
The orbit closure in the sphere is
\[
    \mathcal O_q:=\overline{\{u_i:i\in\mathbb Z\}}=\Phi_q(\mathcal{H}_q),
\]
and Proposition~\ref{prop:rope_phase_law} proves that
\[
    \frac1n\sum_{i=1}^n\delta_{u_i}\Rightarrow \pi_q.
\]
Thus $\pi_q$ is the long-run deterministic distribution of RoPE target directions. If the frequencies are rational multiples of $2\pi$, then the orbit is periodic and $\pi_q$ is uniform on that finite orbit. If $1,\theta_1,\ldots,\theta_L$ are linearly independent over $\mathbb Q$, then $\mathcal{H}_q=\mathbb T^L$, and $m_{\mathcal{H}_q}$ is the uniform measure over $\mathbb{T}^L$.
Intermediate arithmetic relations give the corresponding Haar measure on a proper closed subgroup.

With this notation, the score, D2Q, and weights are
\[
    a_i=\beta_n\langle u_i,x_i\rangle,
    \qquad
    U_i:=1-\langle u_i,x_i\rangle\in[0,2],
\]
and
\[
    A_i:=\frac{e^{-\beta_n U_i}}{\sum_{j=1}^n e^{-\beta_n U_j}},
    \qquad
    A_{(1)}\ge A_{(2)}\ge\cdots\ge A_{(n)}.
\]
Consequently, as in the i.i.d. fixed-query case, the asymptotic behavior of the ordered weights is controlled by the lower tail of the D2Q variables. 
We make the following assumptions in this appendix.
\begin{assumption}[Correlated tokens with RoPE]
\label{ass:rope_correlated}
The sequence $(x_i)_{i\in\mathbb Z}$ is stationary\footnote{A sequence of random variables $(x_k)_{k \in \mathbb{Z}}$ is called
\emph{stationary} if for every $m \geq 1$, every choice of
indices $k_1,\dots,k_m \in \mathbb{Z}$, and every $h \in \mathbb{Z}$,
one has
\[
(x_{k_1},\dots,x_{k_m})
\stackrel{d}{=}
(x_{k_1+h},\dots,x_{k_m+h}).
\]}
with values in $\S$. Moreover:
\begin{enumerate}
    \item[(i)] there exists an integer $m\ge 0$ such that $(x_i)$ is $m$-dependent, i.e., for any two index sets \(I,J\subset\mathbb Z\) with \(\operatorname{dist}(I,J)>m\), we have \(\{x_i:i\in I\}\) and \(\{x_j:j\in J\}\) are independent.
    \item[(ii)] the common law of $x_i$ has a density $\rho$ with respect to surface measure on $\S$, and $\rho$ is continuous and strictly positive on a neighborhood of $\mathcal O_q$.
    \item[(iii)] for each $1\le \ell\le m$, the pair $(x_0,x_\ell)$ has a density with respect to product surface measure, bounded on a neighborhood of $\mathcal O_q\times \mathcal O_q$.
\end{enumerate}
\end{assumption}

\begin{remark}[On stationarity]
Stationarity in Assumption~\ref{ass:rope_correlated} is a convenient sufficient condition, not an essential one. The same argument should extend whenever the rescaled D2Q has the same Poisson Point Process limit.
\end{remark}

\begin{remark}[When anti-clustering fails]
When \Cref{ass:rope_correlated}(iii) fails, the first-order scale often remains $n^{\alpha}$, but the limiting point process is typically no longer Poisson. In that case the limiting ordered weights are obtained by applying the same softmax functional to the appropriate cluster-Poisson extremal process.
\end{remark}

Under \Cref{ass:rope_correlated}, the critical exponent remains
\(
\alpha=\frac{2}{d-1},
\)
and RoPE changes only the constant in the limiting point process through the orbit average $\overline C(q)$.
\[
\overline C(q)
\coloneqq \frac{2^{1/\alpha}\sigma_{d-2}}{d-1}
\int_{\S}\rho(u)\,\pi_q(du),
\]

\begin{theorem}[Critical regime with correlated tokens and RoPE]
\label{thm:rope_critical}
Under Assumption~\ref{ass:rope_correlated}, assume that
\(
\beta_n n^{-\alpha}\longrightarrow \gamma\in(0,\infty) \, ,
\)
and let $\Xp$ be a Poisson point process on $\mathbb R_+:=[0,\infty)$ with intensity measure
\[
\overline{\Lambda}([0,y])=\overline C(q)\Bigl(\frac{y}{\gamma}\Bigr)^{1/\alpha},
\qquad y\ge 0.
\]
Write
\(
0<Y_1<Y_2<\cdots
\)
for the atoms of $\Xp$ arranged in increasing order. Then, for every fixed $k\ge 1$,
\[
\bigl(A_{(1)},\dots,A_{(k)}\bigr)
\Rightarrow
\bigl(W_1,\dots,W_k\bigr)
\qquad\text{as } n\to\infty,
\]
where
\[
W_i:=\frac{e^{-Y_i}}{\sum_{j\ge 1}e^{-Y_j}},
\qquad i\ge 1.
\]
\end{theorem}
\begin{remark}
For simplicity, we state and prove only the critical RoPE limit. Under the same assumptions, the supercritical and subcritical conclusions from the i.i.d.\ case carry over to the RoPE setting. In particular, similar to \Cref{thm:intro super_critical,thm:intro sub_critical 2 new}, \(\beta_n n^{-\alpha}\to\infty\) gives \(A_{(1)}\to1\), while \(\beta_n n^{-\alpha}\to0\) gives \(A_{(i)}\to0\) for every fixed \(i\).
\end{remark}

\begin{remark}
We note that if we denote $F_i(t) = \mathbb{P}(U_i<t)$, then $\overline C(q)$ equals the average of local density along the RoPE orbit:
\[
\overline C(q) = \lim_{n\to\infty} \frac{1}{n}\sum_{i=1}^n \lim_{t\to0^+} F_i(t) t^{-1/\alpha}
\]
In particular, if $\rho$ is invariant under the RoPE rotations along the orbit of $q$, then
\[
\int_{\S}\rho(u)\,\pi_q(du)=\rho(q),
\]
so that $\overline C(q)=C(q)$ and the first-order critical normalization is the
same as in the i.i.d. theorem.
\end{remark}

Define
\[
    \Xp_n := \sum_{i=1}^n \delta_{\beta_n U_i}
    \qquad \text{on } \mathbb R_+ .
\]
We first establish the Poisson limit for \(\Xp_n\).
\begin{lemma}[Poisson point process limit]
\label{lem:rope_ppp}
Assume that
\(
\beta_n n^{-\alpha}\longrightarrow \gamma\in(0,\infty) \ .
\)
Under Assumption~\ref{ass:rope_correlated}, 
we have
\(
\Xp_n\Rightarrow \Xp\)
in
\(
M_p(\mathbb R_+),
\)
where $\Xp$ is a Poisson point process on $\mathbb R_+$ with intensity measure
$\overline \Lambda$ given by
\[
\overline \Lambda([0,y])=\overline C(q)\Bigl(\frac{y}{\gamma}\Bigr)^{1/\alpha},
\qquad y\ge 0.
\]
\end{lemma}

\begin{proof}[Proof of \Cref{lem:rope_ppp}]

We prove the convergence by Laplace functionals.

Fix a nonnegative function $\varphi\in C_c(\mathbb R_+)$. Let $R<\infty$ be such
that $\operatorname{supp}\varphi\subset [0,R]$, and set
\[
G(y):=1-e^{-\varphi(y)},
\qquad
X_{i,n}:=G(\beta_n U_i)=1-e^{-\varphi(\beta_n U_i)}.
\]
Then $0\le X_{i,n}\le \bone_{\{\beta_n U_i\le R\}}$, and
\[
\mathbb E\exp\Bigl(-\int_{\mathbb R_+} \varphi(y)\,\Xp_n(dy)\Bigr)
=\mathbb E\prod_{i=1}^n (1-X_{i,n}).
\]
We will show that this converges to
\[
\exp\Bigl(-\int_{\mathbb R_+} (1-e^{-\varphi(y)})\,\overline \Lambda(dy)\Bigr).
\]

\emph{Step 1}. Uniform one-point tail asymptotics.
Note that by \Cref{ass:rope_correlated}, $\rho$ is continuous on a neighborhood of the compact set $\mathcal O_q$, so it is uniformly continuous there. Hence
\[
\sup_{u\in\mathcal O_q}
\left|
\mathbb P(1-\langle u,x_1\rangle\le t)
-
C(u)t^{1/\alpha}
\right|
=o(t^{1/\alpha}).
\]
Since $u_i\in\mathcal O_q$, this gives, uniformly in $i$,
\begin{equation}
\label{eq:rope_uniform_tail}
\mathbb P(U_i\le t)
=
C(u_i)t^{1/\alpha}
+o(t^{1/\alpha})
\qquad \text{ as } \qquad t \downarrow 0.
\end{equation}
In particular, for every fixed $R<\infty$ there is a constant $C_R<\infty$ such that, for all large $n$ and all $1\le i\le n$,
\begin{equation}
\label{eq:rope_uniform_rare_bound}
\mathbb P(\beta_n U_i\le R)\le \frac{C_R}{n}.
\end{equation}

\emph{Step 2}. Convergence of the mean measures.
Let
\[
\mu_n(B):=\sum_{i=1}^n \mathbb P(\beta_n U_i\in B),
\qquad B\subset\mathbb R_+\text{ Borel}.
\]
For every fixed $y\ge0$, \eqref{eq:rope_uniform_tail} gives
\begin{align}\label{eq:app_rope_mean_measure}
\mu_n([0,y])
=\sum_{i=1}^n \mathbb P\Bigl(U_i\le \frac{y}{\beta_n}\Bigr) 
=
\frac{y^{1/\alpha}}{\beta_n^{1/\alpha}}
\sum_{i=1}^n C(u_i)
+n\,o(\beta_n^{-1/\alpha}) .
\end{align}

Recall from Lemma \ref{lem:small_cap_asymptotics} that, for $u\in\S$
\begin{align}\label{eq:app_Cu}
C(u)=\frac{2^{1/\alpha}\sigma_{d-2}}{d-1}\rho(u) \,.
\end{align}
We need the following lemma for the limit of the mean measure:
\begin{proposition}\label{prop:rope_phase_law}
With $\pi_q=(\Phi_q)_{\#}m_{H_q}$ defined at the beginning of \Cref{app:proof_correlated_critical},
\[
\frac1n\sum_{i=1}^n \delta_{u_i}
\Rightarrow
\pi_q
\qquad\text{weakly on } \mathcal{P}(\S).
\]
\end{proposition}

We defer the proof of Proposition~\ref{prop:rope_phase_law} to the end of this appendix. By Proposition~\ref{prop:rope_phase_law}, we obtain directly that
\begin{align}\label{eq:app_rope_orbit_measure}
\frac1n\sum_{i=1}^n \rho(u_i)
\longrightarrow
\int_{\S}\rho(u)\,\pi_q(du).
\end{align}
Therefore, combining \eqref{eq:app_rope_mean_measure}, \eqref{eq:app_Cu} and \eqref{eq:app_rope_orbit_measure} and using $\beta_n n^{-\alpha}\to\gamma$, we have
\[
\mu_n([0,y])\longrightarrow
\frac{2^{1/\alpha}\sigma_{d-2}}{d-1}
\gamma^{-1/\alpha}y^{1/\alpha}
\int_{\S}\rho(u)\,\pi_q(du)
=\overline C(q)\Bigl(\frac{y}{\gamma}\Bigr)^{1/\alpha}
=\overline \Lambda([0,y]).
\]
The limit distribution function is continuous in $y$, so $\mu_n$ converges vaguely to $\Lambda$ on $\mathbb R_+$. Hence, since $G=1-e^{-\varphi}$ is
continuous and compactly supported,
\begin{equation}
\label{eq:rope_laplace_first_moment}
L_n:=\sum_{i=1}^n \mathbb E X_{i,n}
=\int G(y)\,\mu_n(dy)
\longrightarrow
L:=\int G(y)\,\overline\Lambda(dy).
\end{equation}

\emph{Step 3}. Short-range two-hit bound.
Fix $1\le \ell\le m$. By Assumption~\ref{ass:rope_correlated}(iii), the pair $(x_i,x_{i+\ell})$ has a density bounded on a neighborhood of $\mathcal O_q\times\mathcal O_q$. Therefore, for all sufficiently small $t$,
\[
\mathbb P(U_i\le t,\,U_{i+\ell}\le t)
\le C't^{2/\alpha},
\]
uniformly in $i$ and $1\le \ell\le m$. Taking $t=R/\beta_n$ yields
\begin{equation}
\label{eq:rope_close_pair_bound}
\mathbb E[X_{i,n}X_{i+\ell,n}]
\le
\mathbb P(\beta_nU_i\le R,\,\beta_nU_{i+\ell}\le R)
\le \frac{C_R}{n^2}
\end{equation}
for all large $n$, uniformly in $i$ and $1\le \ell\le m$. 

On the other hand, if $|i-j|>m$, then $m$-dependence in Assumption~\ref{ass:rope_correlated}(i) gives independence, and \eqref{eq:rope_uniform_rare_bound} gives
\begin{equation}
\label{eq:rope_far_pair_bound}
\mathbb E[X_{i,n}X_{j,n}]
=\mathbb E X_{i,n}\,\mathbb E X_{j,n}
\le \frac{C_R}{n^2}.
\end{equation}

\emph{Step 4}. Blocking reduction to independent blocks.
Let $b_n:=\lfloor n^{1/2}\rfloor$. Partition $\{1,\dots,n\}$ into consecutive
large blocks of length $b_n$, separated by gaps of length $m$. More explicitly, set the block
\[
B_{r,n}:=\{(r-1)(b_n+m)+1,\dots,(r-1)(b_n+m)+b_n\}
\]
and we let $\mathcal{I}_n\subset\{1,\dots,n\}$ collect those $r$ such that $B_{r,n} \subset \{1,\dots,n\}$. Let $\mathcal B_n = \bigcup_{r\in\mathcal{I}_n} B_{r,n}$ be the union of these large blocks; see \Cref{fig:app_blocks} for an illustration.

\begin{figure}[htb]
  \centering
\begin{tikzpicture}[
  x=0.40cm,y=0.84cm,
  >=Latex,
  every node/.style={font=\small},
  block/.style={fill=blue!14,draw=blue!65!black,line width=0.55pt,rounded corners=1.0pt},
  gap/.style={fill=orange!25,draw=orange!80!black,line width=0.55pt,rounded corners=1.0pt},
  rem/.style={fill=orange!35,draw=orange!90!black,densely dashed,line width=0.55pt,rounded corners=1.0pt},
  brace/.style={decorate,decoration={brace,amplitude=4pt},line width=0.45pt},
  indextick/.style={gray!70,line width=0.25pt},
  note/.style={font=\scriptsize,align=center},
  legend/.style={font=\scriptsize,inner sep=2pt},
]

% Uniform-index schematic: every integer i is placed at x=i.
% The toy values b_n=5, m=2, n=29 are used only for the drawing.
\def\hh{0.31}

% Intervals on the same equally spaced index line.
\filldraw[block] (0.5,-\hh) rectangle (5.5,\hh);     % B_1
\filldraw[gap]   (5.5,-\hh) rectangle (7.5,\hh);     % gap length m
\filldraw[block] (7.5,-\hh) rectangle (12.5,\hh);    % B_2
\filldraw[gap]   (12.5,-\hh) rectangle (14.5,\hh);   % gap length m
\filldraw[block] (14.5,-\hh) rectangle (19.5,\hh);   % B_r
\filldraw[gap]   (19.5,-\hh) rectangle (21.5,\hh);   % gap length m
\filldraw[block] (21.5,-\hh) rectangle (26.5,\hh);   % last full block
\filldraw[rem]   (26.5,-\hh) rectangle (29.5,\hh);   % final remainder

% Axis and equally spaced integer ticks.
\draw[->,line width=0.55pt] (0.25,0) -- (30.0,0) node[right] {$i$};
\foreach \i in {1,...,29}{
  \draw[indextick] (\i,-\hh) -- (\i,\hh);
  \node[font=\tiny,below=4pt] at (\i,-\hh) {\i};
}

% Labels for blocks and gaps.
\node[above=5pt] at (3,\hh) {$B_{1,n}$};
\node[above=5pt] at (10,\hh) {$B_{2,n}$};
\node[above=5pt] at (17,\hh) {$B_{r,n}$};
\node[above=5pt] at (24,\hh) {$B_{|\mathcal I_n|,n}$};

\node[note,text=orange!80!black,above=5pt] at (6.5,\hh) {gap};
\node[note,text=orange!80!black,above=5pt] at (13.5,\hh) {gap};
\node[note,text=orange!80!black,above=5pt] at (20.5,\hh) {gap};
\node[note,text=orange!90!black,above=5pt] at (28,\hh) {final\\remainder};

% Braces showing the exact lengths in the uniform index coordinate.
\draw[brace] (14.5,1.05) -- (19.5,1.05)
  node[midway,above=5pt] {$b_n$};
\draw[brace] (19.5,1.05) -- (21.5,1.05)
  node[midway,above=5pt] {$m$};

% Membership callouts.
\draw[-{Latex[length=2.1mm]},blue!70!black,line width=0.5pt]
  (16.0,-1.10) -- (17.0,-0.35);
\node[note,text=blue!70!black] at (16.0,-1.38)
  {$i\in B_{r,n}\subset \mathcal B_n$};

\draw[-{Latex[length=2.1mm]},orange!90!black,line width=0.5pt]
  (22.3,-1.10) -- (20.5,-0.35);
\node[note,text=orange!90!black] at (22.3,-1.38)
  {$i\in \mathcal G_n$};

\end{tikzpicture}
  \caption{Large blocks are separated by gaps of length $m$, so the corresponding block variables are independent under $m$-dependence.}
  \label{fig:app_blocks}
\end{figure}
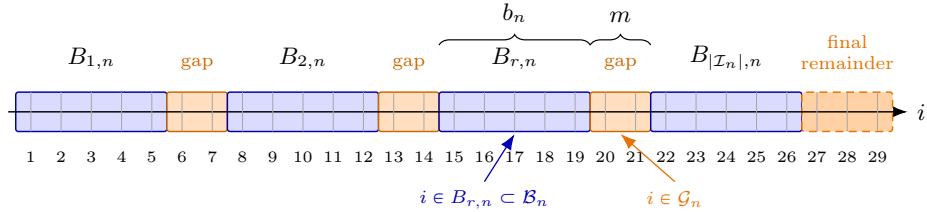

Let $\mathcal G_n:=\{1,\dots,n\}\setminus\mathcal B_n$ be the union of the gaps and the final remainder. Then
\(
|\mathcal G_n|=O\Bigl(\frac{n}{b_n}+b_n\Bigr)=o(n).
\)
Since $0\le X_{i,n}\le1$,
\[
\left|
\prod_{i=1}^n(1-X_{i,n})-
\prod_{i\in\mathcal B_n}(1-X_{i,n})
\right|
\le \sum_{i\in\mathcal G_n}X_{i,n}.
\]
Using \eqref{eq:rope_uniform_rare_bound},
\begin{align}\label{eq:app_neglible_gap}
\mathbb E\sum_{i\in\mathcal G_n}X_{i,n}
\le \frac{C_R|\mathcal G_n|}{n}\longrightarrow0.
\end{align}
Therefore
\begin{equation}
\label{eq:rope_gap_removal}
\mathbb E\prod_{i=1}^n(1-X_{i,n})
=
\mathbb E\prod_{i\in\mathcal B_n}(1-X_{i,n})+o(1).
\end{equation}
The large blocks are separated by more than $m$ indices, so the corresponding
collections of variables are independent. Hence
\[
\mathbb E\prod_{i\in\mathcal B_n}(1-X_{i,n})
=
\prod_r Q_{r,n},
\qquad
Q_{r,n}:=\mathbb E\prod_{i\in B_{r,n}}(1-X_{i,n}).
\]
For each block, define
\(
a_{r,n}:=\sum_{i\in B_{r,n}}\mathbb E X_{i,n}.
\)
The standard Bonferroni bound
\[
\left|\prod_{i\in B}(1-x_i)-\Bigl(1-\sum_{i\in B}x_i\Bigr)\right|
\le \sum_{\substack{i<j\\ i,j\in B}}x_ix_j,
\qquad 0\le x_i\le1,
\]
gives
\[
Q_{r,n}=1-a_{r,n}+\varepsilon_{r,n},
\qquad
|\varepsilon_{r,n}|
\le
\sum_{\substack{i<j\\ i,j\in B_{r,n}}}\mathbb E[X_{i,n}X_{j,n}].
\]
By \eqref{eq:rope_close_pair_bound} and \eqref{eq:rope_far_pair_bound},
\begin{align}\label{eq:app_epsrn_bound}
|\varepsilon_{r,n}|
\le C_R b_n^2 / n^2.
\end{align}
Since the number of large blocks is $O(n/b_n)$,
\begin{equation}
\label{eq:rope_block_error_sum}
\sum_{r\in\mathcal{I}_n} |\varepsilon_{r,n}|
\le C_R\frac{b_n}{n}\longrightarrow0.
\end{equation}
Moreover, by \eqref{eq:rope_uniform_rare_bound},
\begin{align}\label{eq:app_arn_bound}
\max_{r\in\mathcal{I}_n} a_{r,n}\le C_R\frac{b_n}{n}\longrightarrow0,
\end{align}
and by \eqref{eq:rope_laplace_first_moment} together with \eqref{eq:app_neglible_gap},
\[
\sum_{r\in\mathcal{I}_n} a_{r,n}
=\sum_{i\in\mathcal B_n}\mathbb E X_{i,n}
=L_n+o(1)
\longrightarrow L.
\]
By \eqref{eq:app_epsrn_bound} and \eqref{eq:app_arn_bound}, for all large $n$, $| -a_{r,n}+\varepsilon_{r,n}|\le1/2$ uniformly in $r$, and therefore
\[
\sum_{r\in\mathcal{I}_n} \log Q_{r,n}
=
\sum_{r\in\mathcal{I}_n} \log(1-a_{r,n}+\varepsilon_{r,n})
= -\sum_{r\in\mathcal{I}_n} a_{r,n}
+\sum_{r\in\mathcal{I}_n}\varepsilon_{r,n}
+O\Bigl(\sum_{r\in\mathcal{I}_n}(a_{r,n} ^2+\varepsilon_{r,n} ^2)\Bigr).
\]
The last error term tends to zero because
\[
\sum_{r\in\mathcal{I}_n} a_{r,n}^2\le (\max_r a_{r,n})\sum_{r\in\mathcal{I}_n} a_{r,n}\to0,
\qquad
\sum_{r\in\mathcal{I}_n} \varepsilon_{r,n}^2\le (\max_{r\in\mathcal{I}_n}|\varepsilon_{r,n}|)
\sum_{r\in\mathcal{I}_n}|\varepsilon_{r,n}|\to0.
\]
Combining this with
\eqref{eq:rope_block_error_sum} gives
\[
\sum_{r\in\mathcal{I}_n} \log Q_{r,n}\longrightarrow -L.
\]
Consequently,
\begin{equation}
\label{eq:rope_block_product_limit}
\mathbb E\prod_{i\in\mathcal B_n}(1-X_{i,n})
=\prod_{r\in\mathcal{I}_n} Q_{r,n}
\longrightarrow e^{-L}.
\end{equation}

Finally, combining \eqref{eq:rope_gap_removal} and \eqref{eq:rope_block_product_limit}, we get
\[
\mathbb E\exp\Bigl(-\int_{\mathbb R_+} \varphi(y)\,\Xp_n(dy)\Bigr)
\longrightarrow
\exp\Bigl(-\int_{\mathbb R_+}(1-e^{-\varphi(y)})\,\overline\Lambda(dy)\Bigr) ,
\]
which implies the point process convergence
\(
\Xp_n\Rightarrow \Xp
\)
in \(M_p(\mathbb R_+)\).
% This proves the lemma.
\end{proof}

\begin{proof}[Proof of \Cref{prop:rope_phase_law}]

Define the empirical measures on $\mathbb T^L$ by
\[
\eta_n:=\frac1n\sum_{j=1}^n \delta_{j\theta\bmod 1}.
\]
For $k\in\mathbb Z^L$, their Fourier coefficients are
\[
\widehat\eta_n(k)
:=\int_{\mathbb T^L} e^{2\pi \mathrm{i} k\cdot x}\,\eta_n(dx)
=\frac1n\sum_{j=1}^n e^{2\pi \mathrm{i} j k\cdot \theta}.
\]
If $k\cdot\theta\in\mathbb Z$, then every term equals $1$, so $\widehat\eta_n(k)=1$. If $k\cdot\theta\notin\mathbb Z$, then the right-hand side is a normalized geometric sum and therefore converges to $0$. Hence
\[
\widehat\eta_n(k)\longrightarrow \bone_{\{k\cdot\theta\in\mathbb Z\}}.
\]
These are exactly the Fourier coefficients of Haar probability measure $m_{\mathcal{H}_q}$ on $\mathcal{H}_q$ \cite{rudin2017fourier}. Therefore $\eta_n\Rightarrow m_{\mathcal{H}_q}$ on $\mathbb T^L$. Pushing forward by the continuous map $\Phi_q$ yields
\[
\frac1n\sum_{j=1}^n \delta_{u_j}
=(\Phi_q)_{\#}\eta_n
\Rightarrow
(\Phi_q)_{\#}m_{\mathcal{H}_q}
=\pi_q.
\]
This proves the claim.
\end{proof}

\begin{proof}[Proof of \Cref{thm:rope_critical}]
To pass from the point-process limit to the softmax denominator, we use a strengthened Laplace-functional convergence, which essentially follows the proof of \Cref{lem:rope_ppp}. Let
\[
    Z_n:=\int e^{-y}\Xp_n(\mathrm{d}y), \qquad
    Z:=\int e^{-y}\Xp(\mathrm{d}y).
\]
For every \(\varphi\in C_c(\mathbb R_+)\), $\varphi \geq 0$ 
and every \(s\ge0\), set
\[
    G_s(y):=1-\exp\{-\varphi(y)-s e^{-y}\}.
\]
Then
\[
    0\le G_s(y)\le 1-e^{-\varphi(y)}+s e^{-y},
\]
and the right-hand side is integrable against the limiting intensity, since
\(
    \int_0^\infty e^{-y}y^{1/\alpha-1}\,dy<\infty.
\)
The same blocking argument used for the compactly supported Laplace functional in the proof of \Cref{lem:rope_ppp} therefore gives
\[
\begin{aligned}
    \mathbb E\exp\left\{
        -\int \varphi(y)\Xp_n(\mathrm{d}y)-sZ_n
    \right\}
    &=
    \mathbb E\prod_i\bigl(1-G_s(\beta_nU_i)\bigr)  \\
    &\longrightarrow
    \exp\left\{
        -\int G_s(y)\Lambda(\mathrm{d}y)
    \right\}                                      \\
    &=
    \mathbb E\exp\left\{
        -\int \varphi(y)\Xp(\mathrm{d}y)-sZ
    \right\}.
\end{aligned}
\]
This is exactly convergence of the joint Laplace transform of $\left(\Xp_n,\int e^{-y}\Xp_n(\mathrm{d}y)\right)$.
Varying \(\varphi\) over a convergence-determining class in \(C_c(\mathbb R_+)\)
and \(s\ge0\), this gives
\[
    \left(\Xp_n,\int e^{-y}\Xp_n(\mathrm{d}y)\right)
    \Rightarrow
    \left(\Xp,\int e^{-y}\Xp(\mathrm{d}y)\right).
\]

It remains to pass from the point process and its denominator to the ordered
softmax weights. Let
\[
0<Y_{(1)}^n<Y_{(2)}^n<\cdots<Y_{(n)}^n
\]
denote the atoms of \(\Xp_n\) in increasing order, so that
\(Y_{(i)}^n=\beta_n U_{(i)}\) and
\(
A_{(i)}=\tfrac{e^{-Y_{(i)}^n}}{Z_n}.
\)
The limiting intensity measure \(\overline{\Lambda}\) is diffuse, and
\(\overline{\Lambda}([0,R])\to\infty\) as \(R\to\infty\). Hence the limiting Poisson
point process \(\Xp\) is simple, has infinitely many atoms, and its first \(k\)
atoms \(0<Y_{(1)}<\cdots<Y_{(k)}\) are finite almost surely. Moreover
\[
Z=\int e^{-y}\Xp(\mathrm{d}y)=\sum_{j\ge 1}e^{-Y_{(j)}}
\]
is almost surely finite and strictly positive: finiteness follows from
\(\int_0^\infty e^{-y}\overline{\Lambda}(\mathrm{d}y)<\infty\), while positivity follows from the
existence of atoms. Therefore the map
\[
(\mu,z)\longmapsto
\left(
\frac{e^{-y_{(1)}(\mu)}}{z},\ldots,
\frac{e^{-y_{(k)}(\mu)}}{z}
\right),
\]
where \(y_{(i)}(\mu)\) is the \(i\)-th atom of \(\mu\), is continuous at
\((\Xp,Z)\) almost surely. Applying the continuous mapping theorem to the joint
convergence above gives
\[
(A_{(1)},\ldots,A_{(k)})
=
\left(
\frac{e^{-Y_{(1)}^n}}{Z_n},\ldots,
\frac{e^{-Y_{(k)}^n}}{Z_n}
\right)
\Rightarrow
\left(
\frac{e^{-Y_{(1)}}}{\sum_{j\ge1}e^{-Y_{(j)}}},\ldots,
\frac{e^{-Y_{(k)}}}{\sum_{j\ge1}e^{-Y_{(j)}}}
\right).
\]
This is exactly the claimed convergence to \((W_1,\ldots,W_k)\), and proves
Theorem F.3.
\end{proof}

% \end{appendix}

\end{appendix}

\newpage

\addtocontents{toc}{\protect\setcounter{tocdepth}{2}}
\bibliographystyle{alpha}
\bibliography{references}

\end{document}